\documentclass[twoside]{article}

\usepackage[accepted]{aistats2022}
\usepackage[authoryear,round]{natbib}

\usepackage[utf8]{inputenc} 
\usepackage[T1]{fontenc}    
\usepackage{macros}

\renewcommand{\epsilon}{\varepsilon}

\newcommand{\simp}{\mathbf{\Lambda}}

\newcommand{\rage}{\textsf{RAGE}\xspace}
\newcommand{\round}{\textsf{ROUND}\xspace}
\newcommand{\opt}{\textsf{OPT}\xspace}

\newcommand{\gemsc}{\textsf{GEMS-c}\xspace}
\newcommand{\gemsm}{\textsf{GEMS-m}\xspace}
\newcommand{\gemsb}{\textsf{GEMS-b}\xspace}

\begin{document}

%

%

\twocolumn[

\aistatstitle{Near Instance Optimal Model Selection for Pure Exploration Linear Bandits}

\aistatsauthor{ Yinglun Zhu \And Julian Katz-Samuels \And  Robert Nowak }

\aistatsaddress{ University of Wisconsin-Madison \And  University of Wisconsin-Madison \And University of Wisconsin-Madison } ]

\begin{abstract}
We introduce the model selection problem in pure exploration linear bandits, where the learner needs to adapt to the instance-dependent complexity measure of the smallest hypothesis class containing the true model. We design algorithms in both fixed confidence and fixed budget settings with near instance optimal guarantees. The core of our algorithms is a new optimization problem based on experimental design that leverages the geometry of the action set to identify a near-optimal hypothesis class. Our fixed budget algorithm is developed based on a novel selection-validation procedure, which provides a new way to study the understudied fixed budget setting (even without the added challenge of model selection). We adapt our algorithms, in both fixed confidence and fixed budget settings, to problems with model misspecification.
\end{abstract}

\section{INTRODUCTION}
\label{sec:intro}

The pure exploration linear bandit problem considers a set of arms whose expected rewards are linear in their \emph{given} feature representation, and aims to identify the optimal arm through adaptive sampling. 
Two settings, i.e., fixed confidence and fixed budget settings, are studied in the literature.
In the fixed confidence setting, the learner continues sampling arms until a desired confidence level is reached, and the goal is to minimize the total number of samples \citep{soare2014best, xu2018fully, tao2018best, fiez2019sequential, degenne2020gamification, katz2020empirical}. 
In the fixed budget setting, the learner is forced to output a recommendation within a pre-fixed sampling budget, and the goal is to minimize the error probability \citep{hoffman2014correlation, katz2020empirical, alieva2021robust, yang2021towards}.
Applications of pure exploration linear bandits include content recommendation, digital advertisement and A/B/n testing (see aforementioned papers for more discussions on applications).

All existing works, however, focus on linear models with the \emph{given} feature representations and fail to adapt to cases when the problem can be explained with a much simpler model, i.e., a linear model based on a subset of the features. In this paper, we introduce the model selection problem in pure exploration linear bandits. We consider a sequence of nested linear hypothesis classes $\cH_1 \subseteq \cH_2 \subseteq \dots \subseteq \cH_D$ and assume that $\cH_{d_\star}$ is the smallest hypothesis class that contains the true model. Our goal is to automatically adapt to the complexity measure related to $\cH_{d_\star}$, for an unknown $d_\star$, rather than suffering a complexity measure related to the largest hypothesis class $\cH_D$. 

The model selection problem appears ubiquitously in real-world applications. In fact, cross-validation \citep{stone1974cross, stone1978cross}, a practical method for model selection, appears in almost all successful deployments of machine learning models. The model selection problem was recently introduced to the bandit regret minimization setting by \citet{foster2019model}, and further analyzed by \citet{pacchiano2020model, zhu2021pareto}. \citet{zhu2021pareto} prove that only Pareto optimality can be achieved for regret minimization, which is even weaker than minimax optimality. We introduce the model selection problem in the pure exploration setting and, surprisingly, show that it is possible to design algorithms with \emph{near optimal instance-dependent complexity} for both fixed confidence and fixed budget settings. We further generalize the model selection problem to the regime with misspecified linear models, and show our algorithms are robust to model misspecification.

\subsection{Contribution and Outline}
\label{sec:contribution}

We briefly summarize our contributions as follows:

\begin{itemize}[leftmargin=.2in]
    \item We introduce the model selection problem for pure exploration in linear bandits in \cref{sec:setting}, and analyze its instance-dependent complexity measure. We provide a general framework to solve the model selection problem for pure exploration linear bandits. Our framework is based on a carefully-designed two-dimensional doubling trick and a new optimization problem that leverages the geometry of the action set to efficiently identify a near-optimal hypothesis class. 

    \item In \cref{sec:fixed_confidence}, we provide an algorithm for the fixed confidence setting with near optimal instance-dependent unverifiable sample complexity. We additionally provide evidence on why one cannot verifiably output recommendations.
    
    \item In \cref{sec:fixed_budget}, we provide an algorithm for the fixed budget setting, which applies a novel selection-validation trick to bandits. Its probability of error matches (up to logarithmic factors) the probability error of an algorithm that chooses its sampling allocation based on knowledge of the true model parameter. In addition, the guarantee of our algorithm is nearly optimal even in the non-model-selection case, and our algorithm also provides a new way to analyze the \emph{understudied} fixed budget setting.
    
    \item We further generalize the model selection problem into the misspecified regime in \cref{sec:misspecification}, and adapt our algorithms to both the fixed confidence and fixed budget settings. Our algorithms reach an instance-dependent sample complexity measure that is relevant to the complexity measure of a closely related perfect linear bandit problem.
\end{itemize}

\section{PROBLEM SETTING}
\label{sec:setting}

In the transductive linear bandit pure exploration problem, the learner is given an action set $\cX \subset \R^D$ and a target set $\cZ \subset \R^D$. The expected reward of any arm $x \in \cX \cup \cZ$ is linearly parameterized by an unknown reward vector $\theta_\star \in \Theta \subseteq \R^D$, i.e., $h(x) = \ang*{\theta_\star, x}$. The parameter space $\Theta$ is known to the learner. 
At each round $t$, the learner/algorithm $\alg$ selects an action $X_t \in \cX$, and observes a noisy reward $R_t = h(X_t) + \xi_t$, where $\xi_t$ represents an additive $1$-sub-Gaussian noise.
The action $X_t \in \cX$ can be selected with respect to the history $\cF_{t-1} = \sigma((X_i, R_i)_{i<t})$ up to time $t$.
The goal is to identify the unique optimal arm $z_\star = \argmax_{z \in \cZ} h(z)$ from the target set $\cZ$. We assume $\Theta \subseteq \spn (\cX)$ to obtain unbiased estimators for arms in $\cZ$. Without loss of generality, we assume that $\spn(\cX) = \R^D$ (otherwise one can project actions into a lower dimensional space). We further assume that $\spn(\curly*{z_\star - z}_{z \in \cZ}) = \R^D$ for technical reasons. We consider both fixed confidence and fixed budget settings in this paper.

\begin{definition}[Fixed confidence]
\label{def:fixed_confidence}
Fix $\cX, \cZ, \Theta \subseteq \R^D$. An algorithm $\alg$ is called $\delta$-PAC for $(\cX, \cZ, \Theta)$ if (1) the algorithm has a stopping time $\tau$ with respect to $\curly*{\cF_t}_{t\in \N}$ and (2) at time $\tau$ it makes a recommendation $\widehat z \in \cZ$ such that $\P_{\theta_\star} \paren*{ \widehat z = z_\star } \geq 1 - \delta$ for all $\theta_\star \in \Theta$.
\end{definition}

\begin{definition}[Fixed budget]
\label{def:fixed_budget}
Fix $\cX, \cZ, \Theta \subseteq \R^D$ and a budget $T$. A fixed budget algorithm $\alg$ returns a recommendation $\widehat z \in \cZ$ after $T$ rounds.
\end{definition}

\paragraph{The Model Selection Problem.} 
The learner is given a nested sequence of parameter classes $\Theta_1 \subseteq \Theta_2 \subseteq \dots \subseteq \Theta_D$, where $\Theta_{d} \ldef \curly*{\theta \in \R^D: \theta_i = 0 , \forall i > d}$ is the set of parameters such that for any $\theta \in \Theta_{d}$, it only has non-zero entries on its first $d$ coordinates.\footnote{A nested sequence of linear hypothesis classes $\cH_1 \subseteq \cH_2 \subseteq \dots \subseteq \cH_D$ can be constructed based on the nested sequence of parameter classes $\Theta_1 \subseteq \Theta_2 \subseteq \dots \subseteq \Theta_D$, i.e., $\cH_d \ldef \curly{h(\cdot)=\ang{\theta, \cdot}: \theta \in \Theta_d}$.} 
We assume that $\theta_\star \in \Theta_{d_\star}$ for an \emph{unknown} $d_\star$. We call $d_\star$ the intrinsic dimension of the problem and it is set as the index of the smallest parameter space containing the true reward vector. One interpretation of the intrinsic dimension is that only the first $d_\star$ features (of each arm) play a role in predicting the expected reward. Our goal is to automatically adapt to the sample complexity with respect to the intrinsic dimension $d_\star$, rather than suffering from the sample complexity related to the ambient dimension $D$. In the following, we write $\theta_\star \in \Theta_{d_\star}$ to indicate that the problem instance has intrinsic dimension $d_\star$. Besides dealing with the \emph{well-specified} linear bandit problem as defined in this section, we also extend our framework into the \emph{misspecified} setting in \cref{sec:misspecification}, with additional setups introduced therein.

\paragraph{Additional Notations.} For any $x = [x_1, x_2, \dots, x_D]^\top \in \R^D$ and $d \leq D$, we use $\psi_{d}(x) \ldef  [x_1, x_2, \dots, x_{d}]^\top \in \R^{d}$ to denote the truncated feature representation that only keeps its first $d$ coordinates. We also write $\psi_{d}(\cX) \ldef \curly*{\psi_{d}(x): x \in \cX}$ and $\psi_{d}(\cZ) \ldef \curly*{\psi_{d}(z): z \in \cZ}$ to represent the truncated action set and target set, respectively. Note that we necessarily have $\psi_d (\cZ) \subseteq \spn \paren*{\psi_d(\cX)} = \R^{d}$ as long as $\cZ \subseteq \spn \paren*{\cX} = \R^D$. We use $\cY(\psi_d(\cZ)) \ldef \curly*{ \psi_d(z) - \psi_d(z^\prime): z, z^\prime \in \cZ }$ to denote all possible directions formed by subtracted one item from another in $\psi_d(\cZ)$; and use $\cY^\star(\psi_d(\cZ)) \ldef \curly*{ \psi_d(z_\star) - \psi_d(z): z \in \cZ }$ to denote all possible directions with respect to the optimal arm $z_\star$. For any $z \in \cZ$, we use $\Delta_z \ldef h(z_\star) - h(z)$ to denote its sub-optimality gap; we set $\Delta_{\min} \ldef \min_{z \in \cZ \setminus \curly*{z_\star}} \Delta_z$. As in \citet{fiez2019sequential}, we assume $\max_{z \in \cZ} \Delta_z \leq 2$ when analyzing upper bounds. We denote $\cS_k \ldef \curly*{z \in \cZ: \Delta_z < 4 \cdot 2^{-k}}$ (with $\cS_1 \ldef \cZ$). We use $\simp_{\cX} \ldef \curly*{ \lambda \in \R^{\abs{\cX}}: \sum_{x \in \cX} \lambda_x = 1, \lambda_x \geq 0  }$ to denote the $(\abs*{\cX}-1)$-dimensional simplex over actions. For any (continuous) design $\lambda \in \simp_{\cX}$, we use $A_{d}(\lambda) \ldef \sum_{x \in \cX} \lambda_x \, \psi_d(x)  \paren*{\psi_d(x)}^\top \in \R^{d \times d}$ to denote the design matrix with respect to $\lambda$. For any set $\cW \subseteq \R^D$, we denote $\iota(\cW) \ldef \inf_{\lambda \in \simp_{\cX}} \sup_{w \in \cW} \norm{w}^2_{A_d(\lambda)^{-1}}$.\footnote{A generalized inversion is used for singular matrices. See \cref{app:inverse} for detailed discussion.}

\section{TOWARDS THE TRUE SAMPLE COMPLEXITY}
\label{sec:true_complexity}
The instance-dependent sample complexity lower bound for linear bandit is discovered/analyzed in previous papers \citep{soare2014best, fiez2019sequential, degenne2019pure}. We here consider related quantities that take our model selection setting into consideration. For any $d \in [D]$, we define
\begin{align}
\label{eq:rho}
    \rho^\star_{d} \ldef \inf_{\lambda \in \simp_{\cX}}  \sup_{z \in \cZ \setminus \curly*{z_\star }} \frac{\norm*{\psi_d(z_{\star})-\psi_d(z)}^2_{A_{d}(\lambda)^{-1}}}{(h(z_\star) - h(z))^2},
\end{align}
and 
\begin{align}
\label{eq:iota}
    \iota_{d}^\star \ldef \inf_{\lambda \in \simp_{\cX}}  \sup_{z \in \cZ \setminus \curly*{z_\star}} \norm{\psi_d(z_\star) - \psi_d(z)}_{A_{d}(\lambda)^{-1}}^2.
\end{align}

Following analysis in \citet{fiez2019sequential},
we provide a lower bound for the model selection problem $(\cX, \cZ$, $\theta_\star \in \Theta_{d_\star})$ in the fixed confidence setting as follows.

\begin{restatable}{theorem}{thmLowerBoundDeltaPAC}
\label{thm:lower_bound_delta_PAC}
Suppose $\xi_t \sim \cN(0,1)$ for all $t \in \N_+$ and $\delta \in (0, 0.15]$. Any $\delta$-PAC algorithm with respect to $(\cX, \cZ$, $\theta_\star \in \Theta_{d_\star})$ with stopping time $\tau$ satisfies $\E_{\theta_\star} \sq*{\tau} \geq \rho^\star_{d_\star} \log(1/2.4 \delta)$.
\end{restatable}

The above lower bound only works for $\delta$-PAC algorithms, but not for algorithms in the fixed budget setting or with unverifiable sample complexity (see \cref{sec:fixed_confidence}). We now introduce another lower bound for the best possible \emph{non-interactive} algorithm $\alg$. Following the discussion in \citet{katz2020empirical}, we consider any non-interactive algorithm as follows: The algorithm $\alg$ chooses an allocation $\curly*{x_1, x_2, \dots, x_N} \subseteq \cX$ and receive rewards $\curly*{r_1, r_2, \dots, r_N} \subseteq \R$ where $r_i$ is sampled from $\cN(h(x_i), 1)$. The algorithm then recommends $\widehat z = \argmax_{z \in \cZ} \ang*{\widehat \theta_d, z}$ where $\widehat \theta_d = \argmin_{\theta \in \R^d} \sum_{i=1}^N (r_i - \theta^\top \psi_d(x_i))^2$ is the least squares estimator in $\R^d$. The learner is allowed to choose any allocations, \emph{even with the knowledge of $\theta_\star$}, and use any feature mapping such that linearity is preserved, i.e., $d_\star \leq d \leq D$.

\begin{restatable}{theorem}{thmLowerBoundNonInteractive}
\label{thm:lower_bound_non_interactive}
Fix $(\cX, \cZ$, $\theta_\star \in \Theta_{d_\star})$ and $\delta \in (0,0.015]$. Any non-interactive algorithm $\alg$ using a feature mappings of dimension $d \geq d_\star$ makes a mistake with probability at least $\delta$ as long as it uses no more than $\frac{1}{2}  \rho^\star_{d_\star} \log(1/\delta)$ samples.
\end{restatable}

The above lower bound serves as a fairly strong baseline due to the power provided to the non-interactive learner, i.e., the knowledge of $\theta_\star$. 
\cref{thm:lower_bound_non_interactive} indicates (for any non-interactive learner) (1) sample complexity lower bound $\widetilde \Omega(\rho^\star_{d_\star})$ in fixed confidence setting; and (2) error probability lower bound $\Omega(\exp(- T/\rho^\star_{d_\star}))$ in fixed budget setting: Suppose the budget is $T$, one would expect an error probability at least $\Omega(\exp(- T/ \rho^\star_{d_\star}))$ by relating $\frac{1}{2}  \rho^\star_{d_\star} \log(1/\delta)$ to $T$.

Note that all lower bounds are with respect to $\rho^\star_{d_\star}$ rather than $\rho^\star_{d}$ for $d > d_\star$ due to the assumption $\theta_\star \in \Theta_{d_\star}$ for the model selection problem. Our goal is to automatically adapt to the complexity $\rho_{d_\star}^\star$ without knowledge of $d_\star$. The following proposition shows the monotonic relation among $\curly*{\rho_d^\star}_{d = d_\star}^D$.

\begin{restatable}{proposition}{propRhoMonotonic}
\label{prop:rho_monotonic}
The monotonic relation $\rho^\star_{d_1} \leq \rho^\star_{d_2}$ holds true for any $d_\star \leq d_1 \leq d_2 \leq D$.
\end{restatable}
The intuition behind \cref{prop:rho_monotonic} is that the model class $\Theta_{d_2}$ is a superset of $\Theta_{d_1}$ and therefore identifying $z_\star$ in $\Theta_{d_2}$ requires ruling out a larger set of statistical alternatives than in $\Theta_{d_1}$. While \cref{prop:rho_monotonic} is intuitive, its proof is surprisingly technical and involves showing the equivalence of a series of optimization problems.

\subsection{Failure of Standard Approaches}

\begin{restatable}{proposition}{propRhoStarDifferentD}
\label{prop:rho_star_different_d}
For any $\gamma > 0$, there exists an instance $(\cX, \cZ$, $\theta_\star \in \Theta_{d_\star})$ such that $\rho^\star_{d_\star + 1} > \rho^\star_{d_\star} + \gamma$ yet $\iota_{d_\star + 1}^\star \leq 2 \iota_{d_\star}^\star $.
\end{restatable}

One may attempt to solve the model selection problem with a standard doubling trick over dimension, i.e., truncating the feature representations at dimension $d_i = 2^i$ for $i \leq \ceil*{\log_2 D}$ and gradually exploring models with increasing dimension. This approach, however, is directly ruled out by \cref{prop:rho_star_different_d} since such doubling trick could end up with solving a problem with a dimension $d^\prime \leq 2 d_\star$ yet $\rho_{d^\prime}^\star \gg \rho_{d_\star}^\star$. Although doubling trick over dimensions is commonly used to provide \emph{worst-case} guarantees in regret minimization settings \citep{pacchiano2020model, zhu2021pareto}, we emphasize here that matching \emph{instance-dependent} complexities is important in pure exploration setting \citep{soare2014best, fiez2019sequential, katz2020empirical}. Thus, new techniques need to be developed. \cref{prop:rho_star_different_d} also implies that trying to infer the value of $\rho^\star_{d}$ from  $\iota^\star_d$ can be quite misleading. And thus conducting a doubling trick over $\iota^\star_d$ (or an upper bound of it) is likely to fail as well.

\paragraph{Importance of Model Selection.} \cref{prop:rho_star_different_d} also illustrates the importance and necessity of conducting model selection in pure exploration linear bandits. Consider the hard instance used in constructed in \cref{prop:rho_star_different_d} and set $D= d_\star + 1$. All existing algorithms \citep{soare2014best, fiez2019sequential, degenne2019pure, katz2020empirical} that directly work with the \emph{given} feature representation in $\R^D$ end up with a complexity measure scales with $\rho^\star_{D}$, which could be arbitrarily large than the true complexity measure $\rho_{d_\star}^\star$ and even become vacuous (by sending $\gamma \rightarrow \infty$).

\paragraph{Our Approaches.} In this paper, we design a more sophisticated doubling scheme over a two-dimensional grid corresponding to the number of elimination steps and the richest hypothesis class considered at each step. We design subroutines for both fixed confidence and fixed budget settings. Our algorithms define a new optimization problem based on experimental design that leverages the geometry of the action set to efficiently identify a near-optimal hypothesis class. Our fixed budget algorithm additionally uses a novel application of a selection-validation trick in bandits. Our guarantees are with respect to the true instance-dependent complexity measure $\rho_{d_\star}^\star$.

\section{FIXED CONFIDENCE SETTING}
\label{sec:fixed_confidence}

We present our main algorithm (\cref{alg:doubling_fixed_confidence}) for the fixed confidence setting in this section.
\cref{alg:doubling_fixed_confidence} invokes \gemsc (\cref{alg:subroutine_fixed_confidence}) as subroutines and starts to output the optimal arm after $\widetilde O(\rho^\star_{d_\star} + d_\star)$ samples.
Our sample complexity matches, up to an additive $d_\star$ term and logarithmic factors, the strong baseline developed in \cref{thm:lower_bound_non_interactive}.

We first introduce the subroutine \gemsc, which runs for $n$ rounds and takes (roughly) $B$ samples per-round.
\gemsc is built on \rage \citep{fiez2019sequential}, a standard linear bandit pure exploration algorithm works in the ambient space $\R^D$.
The key innovation of \gemsc lies in \emph{adaptive} hypothesis class selection at each round (i.e., selecting $d_k$), which allows us to adapt to the instrinsic dimension $d_\star$.
After selecting the working dimension $d_k$ at round $k$, \gemsc allocates samples based on optimal design (in $\R^{d_k}$); it then eliminate sub-optimal arms based on the estimated rewards constructed using least squares.
Following \citet{fiez2019sequential}, we use a rounding procedure $\round (\lambda,N,d,\zeta)$ to round a continuous experimental design $\lambda \in \simp_{\cX}$ into integer allocations over actions. 
We use $r_d(\zeta)$ to denote the number of samples needed for such rounding in $\R^d$ with approximation factor $\zeta$. 
One can choose $r_d(\zeta) = (d^2+d+2)/\zeta$ \citep{pukelsheim2006optimal, fiez2019sequential} or $r_d(\zeta) = 180d/\zeta^2$ \citep{allen2020near}.
We choose $\zeta$ as a constant throughout the paper, e.g., $\zeta = 1$.
When $N \geq r_d(\zeta)$, there exist computationally efficient rounding procedures that output an allocation $\curly{x_1, x_2, \dots, x_N}$ satisfying
\begin{align}
    & \max_{y \in \cY(\psi_d(\cZ))}  \norm{y}^2_{\paren{\sum_{i=1}^N \psi_d(x_i) \psi_d(x_i)^{\top}}^{-1}} \leq  \nonumber \\
    &(1+\zeta) \max_{y \in \cY(\psi_d(\cZ))} \norm{y}^2_{\paren{\sum_{x \in \cX} \lambda_x \psi_d(x) \psi_d(x)^\top}^{-1}} / N. \label{eq:rounding}
\end{align}

\begin{algorithm}[]
	\caption{\gemsc Gap Elimination with Model Selection (Fixed Confidence)}
	\label{alg:subroutine_fixed_confidence} 
	\renewcommand{\algorithmicrequire}{\textbf{Input:}}
	\renewcommand{\algorithmicensure}{\textbf{Output:}}
	\begin{algorithmic}[1]
		\REQUIRE Number of iterations $n$, budget for dimension selection $B$ and confidence parameter $\delta$.
		\STATE Set $\widehat \cS_1 = \cZ$.
		\FOR {$k = 1, 2, \dots, n$}
		\STATE Set $\delta_k = \delta/k^2$.
		\STATE Define $g_k(d) \ldef \max \curly*{ 2^{2k} \, \iota(\cY(\psi_d(\widehat \cS_k))), r_d(\zeta) }$. 
		\STATE Get $d_k = \text{\opt}(B, D, g_k(\cdot))$, where $d_k \leq D$ is largest dimension such that $g_k(d_k) \leq B$ (see \cref{eq:opt_d_selection} for the detailed optimization problem); set $\lambda_k$ be the optimal design of the optimization problem\\
		$\inf_{\lambda \in \simp_{\cX}} \sup_{z, z^\prime \in \widehat \cS_k} \norm*{\psi_{d_k}(z)-\psi_{d_k}(z^\prime)}^2_{A_{d_k}(\lambda)^{-1}}$;\\
		set $ N_k = \ceil*{g(d_k) 2(1 + \zeta)  \log(\abs*{\widehat \cS_k}^2/\delta_k)}.$
		\STATE Get allocation \\
		$\curly*{x_1, \ldots, x_{N_k} } = \text{\round}(\lambda_k,N_k, d_k, \zeta)$.
		\STATE Pull arms $\curly*{x_1, \ldots, x_{N_k}} $ and receive rewards $\curly*{r_1, \ldots, r_{N_k}}$.
        \STATE Set $\widehat{\theta}_k = A_k^{-1} b_k \in \R^{d_k}$, \\
        where $A_k = \sum_{i=1}^{N_k} \psi_{d_k}(x_i) \psi_{d_k}(x_i)^\top$, \\
        and $b_k =  \sum_{i=1}^{N_k} \psi_{d_k}(x_i) b_i$.
        \STATE Set 
        $\widehat \cS_{k+1} = \widehat \cS_k \setminus \{z \in \widehat \cS_k : 
        \exists z^\prime \text{ s.t. } \ang*{\widehat{\theta}_k, \psi_{d_k}(z^\prime) - \psi_{d_k}(z) } \geq \omega(z^\prime, z) \}$, where $\omega(z^\prime, z) \ldef \norm{\psi_{d_k}(z^\prime) - \psi_{d_k}(z)}_{A_k^{-1}} \sqrt{2 \log \paren*{ {\abs*{\widehat \cS_k}^2}/{\delta_k} }}$.
		\ENDFOR 
		\ENSURE Set of uneliminated arms $\widehat \cS_{n+1}$.
	\end{algorithmic}
\end{algorithm}

We now discuss 
the adaptive selection of hypothesis class, which is achieved through a new optimization problem: 
At round $k$, $d_k \in [D]$ is selected as the largest dimension such that the value of an experimental design is no larger than the fixed selection budget $B$, i.e.,
\begin{align}
    & \max d  \label{eq:opt_d_selection} \\
    & \text{ s.t. } d \in [D], \nonumber \\
    & \qquad \max \curly*{2^{2k} \cdot \inf_{\lambda \in \bLambda_{\cX}} \sup_{y \in \cY(\psi_d(\widehat \cS_k))} \norm{y}^2_{A_d(\lambda)^{-1}} , r_d(\zeta) }\leq B. \nonumber
\end{align}
The experimental design leverages the geometry of the \emph{uneliminated} set of arms. Intuitively, the algorithm is selecting the \emph{richest} hypothesis class that still allows the learner to improve its estimates of the gaps by a factor of 2 using (roughly) $B$ samples. 
When the budget for dimension selection $B$ is large enough, \gemsc operates on well-specified linear bandits (i.e., using $d_k \geq d_\star$) at all rounds, guaranteeing that the output set of arms are $(2^{1-n})$-optimal. The next lemma provides guarantees for \gemsc.

\begin{restatable}{lemma}{lmSubroutineFixedConfidence}
\label{lm:subroutine_fixed_confidence}
Suppose $B \geq \max \curly*{64 \rho_{d_\star}^\star, r_{d_\star}(\zeta)}$. With probability at least $1-\delta$, \gemsc outputs a set of arms $\widehat \cS_{n+1}$ such that $\Delta_z < 2^{1-n}$ for any $z \in \widehat \cS_{n+1}$.
\end{restatable}

\begin{algorithm}[]
	\caption{Adaptive Strategy for Model Selection (Fixed Confidence)}
	\label{alg:doubling_fixed_confidence} 
	\renewcommand{\algorithmicrequire}{\textbf{Input:}}
	\renewcommand{\algorithmicensure}{\textbf{Output:}}
	\newcommand{\algorithmicbreak}{\textbf{break}}
    \newcommand{\BREAK}{\STATE \algorithmicbreak}
	\begin{algorithmic}[1]
		\REQUIRE Confidence parameter $\delta$.
		\STATE Randomly select a $\widehat z_\star \in \cZ$ as the recommendation for the optimal arm.
		\FOR {$\ell = 1, 2, \dots$}
		\STATE Set $\gamma_\ell = 2^\ell$ and $\delta_\ell = \delta/(2\ell^3)$.
		    \FOR {$i = 1, 2, \dots, \ell$}
		    \STATE Set $n_i  =2^i$, $B_i = \gamma_{\ell}/n_i= 2^{\ell -i}$, and \\
		    get $\widehat \cS_i = \text{\gemsc}(n_i, B_i, \delta_\ell)$.
		    \IF{$\widehat \cS_i = \curly*{\widehat z}$ is a singleton set}
		    \STATE Update the recommendation $\widehat z_\star = \widehat z$. 
		    \BREAK  \, (the inner for loop over $i$)
		    \ENDIF
		    \ENDFOR
		\ENDFOR 
	\end{algorithmic}
\end{algorithm}

We present our main algorithm for model selection in \cref{alg:doubling_fixed_confidence}, which loops over an iterate $\ell$ with roughly geometrically increasing budget $\gamma_\ell = \ell 2^\ell$. Within each iteration $\ell$, \cref{alg:doubling_fixed_confidence} invokes \gemsc $\ell$ times with different configurations $(n_i, B_i)$: $n_i$ is viewed as a guess for the unknown quantity $\log_2(1/\Delta_{\min})$; and $B_i$ is viewed as a guess of $\rho^\star_{d_\star}$, which is then used to determine the adaptive selection hypothesis class. The configurations $\curly*{(n_i, B_i)}_{i=1}^\ell$ are chosen as the diagonal of a two dimensional gird over $n_i$ and $B_i$. 
Within each iteration $\ell$, the recommendation $\widehat z_\star$ is updated as the arm contained in the \emph{first} singleton set returned (if any). Since $B_i$ is chosen in a decreasing order, we are recommending the arm selected from the richest hypothesis class that terminates recommending a single arm. The singleton is guaranteed to contain the optimal arm once a rich enough hypothesis class is considered.
We provide the formal guarantees as follows.\looseness=-1

\begin{restatable}{theorem}{thmDoublingFixedConfidence}
\label{thm:doubling_fixed_confidence}
Let $\tau_\star = \log_2(4/\Delta_{\min}) \max \curly*{\rho^\star_{d_\star}, r_{d_\star}(\zeta)}$.
With probability at least $1-\delta$, \cref{alg:doubling_fixed_confidence} starts to output the optimal arm within iteration $\ell_\star = O( \log_2(\tau_\star))$, and takes at most $N = O \paren{ \tau_\star \log_2(\tau_\star) \log(\abs{\cZ} \log_2(\tau_\star)/\delta) }$ samples.
\end{restatable}

The sample complexity in \cref{thm:doubling_fixed_confidence} is analyzed in an unverifiable way:
\cref{alg:doubling_fixed_confidence} starts to output the optimal arm after $N$ samples, but it does not stop its sampling process.
Nevertheless, up to a rounding-related term and other logarithmic factors,\footnote{We refer readers to \citet{katz2020true} for detailed discussion on unverifiable sample complexity. The rounding term $r_{d_\star}(\zeta) = O(d_\star/\zeta^2)$ commonly appears in the linear bandit pure exploration literature \citep{fiez2019sequential, katz2020empirical}. Although we do not focus on optimizing logarithmic terms in this paper, e.g., the $\log(\abs{\cZ})$ term, our techniques can be extended to address this by combining techniques developed in \citet{katz2020empirical}.} the unverifiable sample complexity matches the non-interactive lower bound developed in \cref{thm:lower_bound_non_interactive}.
The non-interactive lower bound serves as a fairly strong baseline since the non-interactive learner is allowed to sample \emph{with the knowledge of $\theta_\star$}.
Computationally, \cref{alg:doubling_fixed_confidence} starts to output the optimal arm after iteration $\ell_\star$, with at most $O(\ell_\star^2)$ subroutines (\cref{alg:subroutine_fixed_confidence}) invoked.
At each iteration $\ell \leq \ell_\star$, \cref{alg:subroutine_fixed_confidence} is invoked with configurations $n_i$, $B_i$ such that $n_i B_i = 2^\ell \leq 2^{\ell_\star}$ (note that $\ell_\star$ is of logarithmic order).
Up to a model selection step (i.e., selecting $d_k$),
the per-round computational complexity of \cref{alg:subroutine_fixed_confidence} is similar to the complexity of the standard linear bandit algorithm \rage.

\paragraph{Why Not Recommend Arm Verifiably?} We provide a simple example to demonstrate that outputting the estimated best arm (using least squares) before examining full vectors in $\R^D$ can lead to incorrect answers, indicating that verifiable sample complexity, i.e., the number of samples required to terminate the game with a recommendation, scales with $D$ ($\rho^\star_D$). We consider a linear bandit problem with action set $\cX = \cZ = \curly*{e_i}_{i=1}^{D}$. We consider two cases: either (1) $\theta_\star  \ldef [1, 0, \dots, 0, 0]^\top \in \R^D$ with $z_\star = e_1$; or (2) $\theta_\star  \ldef [1, 0, \dots, 0, 2]^\top \in \R^D$ with $z_\star = e_D$. We assume \emph{deterministic} feedback in this example. Let $n_x \geq 1$ denote the number of pulls on arm $x \in \cX$. In both cases, for any $d < D$, the design matrix $\sum_{x \in \cX} n_x \psi_d(x) \psi_d(x)^\top$ is diagonal with entries $(n_{e_i})_{i=1}^{d}$, and the least squares estimator is $\widehat \theta_d = e_1 \in \R^d$. As a result, $e_1$ will be recommended as the best arm: the recommendation is correct in the first case but incorrect in the second case. Essentially, one cannot rule out the possibility that $d_\star$ is equal to $D$ without examining full vectors in $\R^D$. 
Verifiably identifying the best arm in $\R^D$ (with noisy feedback) takes $\widetilde \Omega(\rho^\star_D)$ samples \citep{fiez2019sequential}.

\section{FIXED BUDGET SETTING}
\label{sec:fixed_budget}

We study the fixed budget setting with $\cZ \subseteq \cX$, which includes the linear bandit problem $\cZ = \cX$ as a special case. 
Similar to fixed confidence setting, we develop a main algorithm (\cref{alg:doubling_fixed_budget}) that invokes a base algorithm as subroutines (\gemsb, \cref{alg:subroutine_fixed_budget}). 
\cref{alg:doubling_fixed_budget} achieves an error probability $\widetilde O(\exp(-T/\rho^\star_{d_\star}))$, which, again, matches the strong baseline developed in \cref{thm:lower_bound_non_interactive}.

\begin{algorithm}[]
    \caption{\gemsb Gap Elimination with Model Selection (Fixed Budget)}
    \label{alg:subroutine_fixed_budget} 
	\renewcommand{\algorithmicrequire}{\textbf{Input:}}
	\renewcommand{\algorithmicensure}{\textbf{Output:}}
	\begin{algorithmic}[1]
	\REQUIRE Total budget $T$ (allowing non-integer input), number of rounds $n$, budget for dimension selection $B$.
	\STATE Set $T^\prime = \floor*{T/n}$, $\widehat \cS_1 = \cZ$. Set $\widetilde D$ as the largest dimension that ensures rounding with $T^\prime$ samples, i.e., $\widetilde D = \text{\opt}(T^\prime, D, f(\cdot))$, where $f(d) = r_d(\zeta)$.
	\FOR {$k = 1,  \dots, n$}
	\STATE Define function $g_k(d) \ldef 2^{2k} \, \iota(\cY(\psi_d(\widehat \cS_k)))$. 
	\STATE Get $d_{k} = \text{\opt}(B, \widetilde D, g_k(\cdot) )$, where where $d_k \leq \widetilde  D$ is largest dimension such that $g_k(d_k) \leq B$ (similar to the optimization problem in \cref{eq:opt_d_selection}). Set $\lambda_k$ be the optimal design of the optimization problem \\
	$\inf_{\lambda \in \simp_{\cX}} \sup_{z, z^\prime \in \widehat \cS_k} \norm*{\psi_{d_k}(z)-\psi_{d_k}(z^\prime)}^2_{A_{d_k}(\lambda)^{-1}}$.
	\STATE Get allocations \\
	$\{x_1, \ldots, x_{T^\prime} \} = \text{\round} (\lambda_{k},T^\prime,d_k, \zeta)$.
	\STATE Pull arms $\curly*{x_1, \ldots, x_{T^\prime}} $ and receive rewards $\curly*{r_1, \ldots, r_{T^\prime}}$.
	\STATE Set $\widehat{\theta}_k = A_k^{-1} b_k \in \R^{d_k}$, \\
	where $A_k = \sum_{i=1}^{N_k} \psi_{d_k}(x_i) \psi_{d_k}(x_i)^\top$,\\
	and $b_k =  \sum_{i=1}^{N_k} \psi_{d_k}(x_i) b_i$.
    \STATE Set $\widehat \cS_{k+1} = \widehat \cS_k \setminus \{z \in \widehat \cS_k : \exists z^\prime \text{ s.t. } \ang*{\widehat{\theta}_k, \psi_{d_k}(z^\prime) - \psi_{d_k}(z) } \geq 2^{-k} \}$.
	\ENDFOR
	\ENSURE Any uneliminated arm $\widehat z_\star \in \widehat \cS_{n+1}$.
	\end{algorithmic}
\end{algorithm}

The subroutine \gemsb takes sample budget $T$, number of iterations $n$ and dimension selection budget $B$ as input, and outputs an (arbitrary) uneliminated arm after $n$ iterations. As in the fixed confidence setting, \gemsb performs adaptive selection of the hypothesis class through an optimization problem defined similar to the one in \cref{eq:opt_d_selection}. The main differences from the fixed confidence subroutine is as follows: the selection budget $B$ is only used for dimension selection, and the number of samples allocated per iteration is determined as $\floor*{T/n}$. \gemsb is guaranteed to output the optimal arm with probability $1 - \widetilde O (\exp(- T/\rho^\star_{d_\star}))$ when the selection budget $B$ is selected properly, as detailed in \cref{lm:subroutine_fixed_budget}.

\begin{restatable}{lemma}{lmSubroutineFixedBudget}
\label{lm:subroutine_fixed_budget}
Suppose $64 \rho_{d_\star}^\star \leq B \leq 128 \rho_{d_\star}^\star $ and $T/n \geq r_{d_\star}(\zeta) + 1$. \cref{alg:subroutine_fixed_budget} outputs an arm $\widehat z_\star$ such that $\Delta_{\widehat z_\star} < 2^{1-n}$ with probability at least
\begin{align*}
    1- n \abs*{ \cZ }^2 \exp \paren{ - { T}/{ 640 \, n \, \rho_{d_\star}^\star } }.
\end{align*}
\end{restatable}

\begin{algorithm}[]
	\caption{Adaptive Strategy for Model Selection (Fixed Budget)}
	\label{alg:doubling_fixed_budget} 
	\renewcommand{\algorithmicrequire}{\textbf{Input:}}
	\renewcommand{\algorithmicensure}{\textbf{Output:}}
	\begin{algorithmic}[1]
		\REQUIRE Total budget $2T$.
		\STATE \textbf{Step 1: Selection.} Initialize an empty selection set $\cA = \emptyset$.
		\STATE Set $p = \floor*{W(T)}$ and $T^\prime = {T/p}$.
		\FOR {$i = 1, \dots, p$}
		\STATE Set $B_i = 2^i$, $q_i = \floor*{W(T^\prime/B_i)}$ and $T^{\prime \prime} = {T^\prime/q_i}$.
		\FOR {$j = 1,\dots, q_i$}
		\STATE Set $n_j = 2^j$. \\
		Get $\widehat z_\star^{ij} = \text{\gemsb}(T^{\prime \prime}, n_j, B_i)$ and insert $\widehat z_\star^{ij}$ into the pre-selection set $\cA$.
		\ENDFOR
		\ENDFOR
		\STATE \textbf{Step 2: Validation.} Pull each arm in the pre-selection set $\cA$ exactly $\floor*{T/\abs*{\cA}}$ times.
		\ENSURE Output arm $\widehat z_\star$ with the highest empirical reward from the validation step.
	\end{algorithmic}
\end{algorithm}

Our main algorithm for the fixed budget setting is introduced in \cref{alg:doubling_fixed_budget}. \cref{alg:doubling_fixed_budget} consists of two phases: a pre-selection phase and a validation phase. The pre-selection phase collects a set of potentially optimal arms, selected by subroutines, and the validation phase examines the optimality of the collected arms. We provide \cref{alg:doubling_fixed_budget} with $2T$ total sample budget, and split the budget equally for each phase. At least one good subroutine is guaranteed to be invoked in the pre-selection phase (for sufficiently large $T$). The validation step focuses on identifying the best arm among the pre-selected $O((\log_2 T)^2)$ candidates (as explained in the next paragraph). Our selection-validation trick can be viewed as a \emph{dimension-reduction} technique: we convert a linear bandit problem in $\R^D$ (with unknown $d_\star$) to another linear bandit problem in $\R^{O((\log_2 T)^2)}$,\footnote{Technically, we treat the problem as a standard multi-armed bandit problem with $O((\log_2 T)^2)$ arms, which is a special case of a linear bandit problem in $\R^{O((\log_2 T)^2)}$.} i.e., a problem whose dimension is only polylogarithmic in the budget $T$. 

For non-negative variable $p$, we use $p = W(T)$ to represent the solution of equation $T = p \cdot 2^p$. One can see that $W(T) \leq \log_2 T$. As a result, at most $(\log_2 T)^2$ subroutines are invoked with different configurations of $\curly*{(T^{\prime \prime}, n_j, B_i)}$. The use of $W(\cdot)$ is to make sure that $T^{\prime \prime} \geq n_j B_i $ for all subroutines invoked. This provides more efficient use of budget since the error probability upper bound guaranteed by \gemsb scales as $\widetilde O (\exp(- T^{\prime \prime} / n_j B_i ))$.

\begin{restatable}{theorem}{thmDoublingFixedBudget}
\label{thm:doubling_fixed_budget}
Suppose $\cZ \subseteq \cX$. If $T = \widetilde \Omega \paren{ \log_2(1/\Delta_{\min}) \max \curly*{\rho_{d_\star}^
    \star, r_{d_\star}(\zeta)} }$, then \cref{alg:doubling_fixed_budget} outputs the optimal arm with error probability at most
    \begin{align*}
    &\log_2 (4/\Delta_{\min}) \abs*{ \cZ }^2 \exp \paren{ - \frac{ T}{ 1024 \, \log_2 (4/\Delta_{\min}) \, \rho_{d_\star}^\star } } \\
    & \quad + 2 (\log_2 T)^2 \exp \paren{ - \frac{T}{8 (\log_2 T)^2/ \Delta_{\min}^2} }.
\end{align*}
Furthermore, if there exist universal constants such that $\max_{x \in \cX} \norm{\psi_{d_\star}(x)}^2 \leq c_1$ and $\min_{z \in \cZ} \norm{\psi_{d_\star}(z_\star) - \psi_{d_\star}(z)}^2 \geq c_2$, the error probability is upper bounded by
    \begin{align*}
    O \Bigg(& \max \curly*{ \log_2(1/\Delta_{\min}) \abs*{\cZ}^2,  (\log_2 T)^2 } \\
    & \times  \exp \paren{ - \frac{c_2 T}{\max \curly*{ \log_2(1/\Delta_{\min}), (\log_2 T)^2 } c_1 \rho_{d_\star}^\star} }  \Bigg).
\end{align*}
\end{restatable}

Under the mild assumption discussed above, the error probability of \cref{alg:doubling_fixed_budget} scales as $\widetilde O(\exp(-T/\rho^\star_{d_\star}))$. Such an error probability not only matches, up to logarithmic factors, the strong baseline developed in \cref{thm:lower_bound_non_interactive}, but also matches the error bound in the non-model-selection setting (with known $d_\star$) \citep{katz2020empirical} (Algorithm 3 therein, which is also analyzed under a mild assumption).  
Computationally, \cref{alg:doubling_fixed_budget} invokes \cref{alg:subroutine_fixed_budget} at most $(\log_2 T)^2$ times, each with budget $T^{\prime \prime} \leq T$ and $n_j, B_i$ such that $n_j B_i \leq T$. 
The per-round computational complexity of \cref{alg:subroutine_fixed_confidence} is similar to the one of \cref{alg:subroutine_fixed_budget} (with similar configurations).

Compared to the fixed confidence setting, the fixed budget setting in linear bandits is relatively less studied \citep{hoffman2014correlation, katz2020empirical, alieva2021robust, yang2021towards}. 
To our knowledge, even without the added challenge of model selection, near \emph{instance optimal} error probability guarantee is only achieved by Algorithm 3 in
\citet{katz2020empirical}.
Our \cref{alg:doubling_fixed_budget} provides an alternative way to tackle the fixed budget setting, through a novel selection-validation procedure.
Our techniques might be of independent interest.

\section{MODEL SELECTION WITH MISSPECIFICATION}
\label{sec:misspecification}

We generalize the model selection problem into the \emph{misspecified} regime in this section. Our goal here is to identify an $\epsilon$-optimal arm due to misspecification. We aim to provide sample complexity/error probability guarantees with respect to a hypothesis class that is rich enough to allow us to identify an $\epsilon$-optimal arm.
Pure exploration with model misspecification are recently studied in the literature \citep{alieva2021robust, camilleri2021high, zhu2021pure}.
The model selection criterion we consider here further complicates the problem setting and are not covered in previous work.

We consider the case where the expected reward $h(x)$ of any arm $x \in \cX \cup \cZ \subseteq \R^D$ cannot be perfectly represented as a linear model in terms of its feature representation $x$. We use function $\widetilde \gamma (d)$ to capture the misspecification level with respect to truncation the level $d \in [D]$, i.e.,
\begin{align}
   \widetilde \gamma(d) \ldef \min_{\theta \in \R^D} \max_{x \in \cX \cup \cZ} \abs*{h(x) - \ang*{\psi_d(\theta), \psi_d(x)}}. \label{eq:mis_level}
\end{align}
We use $\theta^d_\star \in \argmin_{\theta \in \R^D} \max_{x \in \cX \cup \cZ} \abs*{h(x) - \ang*{\psi_d(\theta), \psi_d(x)}}$ to denote (any) reward parameter that best captures the worst case deviation in $\R^d$, and use $\eta_d(x) \ldef h(x) - \ang*{\psi_d(\theta_\star^d), \psi_d(x)}$ to represent the corresponding misspecification with respect to arm $x \in \cX \cup \cZ$. We have $\max_{x \in \cX \cup \cZ} \abs*{\eta_d(x)} \leq \widetilde \gamma(d)$ by definition. Although the value of $\eta_d(x)$ depends on the selection of the possibly non-unique $\theta^d_\star$, only the worst-case deviation $\widetilde \gamma(d)$ is used in our analysis. Our results in this section are mainly developed in cases when $\cZ \subseteq \cX$, which contains the linear bandit problem $\cZ = \cX$ as a special case. 

\begin{restatable}{proposition}{propNonIncreasingMisspecification}
\label{prop:non_increasing_misspecification}
The misspecification level $\widetilde \gamma(d)$ is non-increasing with respect to $d$.
\end{restatable}

The non-increasing property of $\widetilde \gamma (d)$ reflect the fact that the representation power of the linear component is getting better in higher dimensions. Following \citet{zhu2021pure}, we use $\gamma(d)$ to quantify the sub-optimality gap of the identified arm, i.e.,
\begin{align*}
    \gamma(d) \ldef & \min \Big\{ 2 \cdot 2^{-n}: n \in \N, \forall k \leq n, \\
    & \paren{2 + \sqrt{(1+\zeta) \iota \paren*{ \cY( \psi_d(\cS_k)) }} } \widetilde \gamma(d) \leq 2^{-k}/2 \Big\}.
\end{align*}

It can be shown that, for any fixed $d \in [D]$, at least a $O(\sqrt{d} \, \widetilde \gamma (d))$-optimal arm can be identified in the existence of misspecification. Such inflation from $\widetilde \gamma(d)$ to $\sqrt{d}\, \widetilde \gamma(d)$ is unavoidable in general: \citet{lattimore2020learning} constructs a hard instance such that identifying a $o(\sqrt{d}\widetilde \gamma(d))$-optimal arm requires sample complexity exponential in $d$, even with \emph{deterministic} feedback. On the other hand, identifying a $\Omega(\sqrt{d} \, \widetilde \gamma(d))$-optimal arm only requires sample complexity polynomial in $d$. Such a sharp tradeoff between sample complexity and achievable optimality motivates our definition of $\gamma(d)$.

We assume $\gamma(d)$ can be made arbitrarily small for $d\in[D]$ large enough, which includes instances with no misspecification in $\R^D$ as special cases.\footnote{We make this assumption in order to identify an $\epsilon$-optimal arm for any pre-defined $\epsilon > 0$. Otherwise, one can adjust the goal and identify arms with appropriate sub-optimality gaps.} For any $\epsilon > 0$, we define 
$d_\star(\epsilon) \ldef \min \curly*{d \in [D]: \forall d^\prime \geq d, \gamma(d^\prime) \leq \epsilon }$.
We aim at identifying an $\epsilon$-optimal arm with sample complexity related to $\rho_{d_\star(\epsilon)}^\star$, which is defined as an $\epsilon$-relaxed version of complexity measure $\rho_{d_\star}^\star$, i.e.,
\begin{align*}
    \rho^\star_{d}(\epsilon) \ldef \inf_{\lambda \in \simp_{\cX}} \sup_{z \in \cZ \setminus \curly*{z_\star }} \frac{\norm*{\psi_d(z_{\star})-\psi_d(z)}^2_{A_{d}(\lambda)^{-1}}}{(\max \curly*{ h(z_\star) - h(z), \epsilon })^2}.
\end{align*}
We consider a closely related complexity measure $\widetilde \rho_{d}^\star(\epsilon)$, which is defined with respect to linear component $\widetilde h(x) \ldef \ang*{\psi_d(\theta_\star^d), \psi_d(x)}$, i.e.,
\begin{align*}
    & \widetilde \rho^\star_{d}(\epsilon) \ldef \\
    &\inf_{\lambda \in \simp_{\cX}} \sup_{z \in \cZ \setminus \curly*{z_\star }} \frac{\norm*{\psi_d(z_{\star})-\psi_d(z)}^2_{A_{d}(\lambda)^{-1}}}{(\max \curly*{\ang*{ \psi_{d}(\theta_\star^d), \psi_{d}(z_\star) - \psi_d(z) }, \epsilon })^2}.
\end{align*}
\begin{restatable}[\citet{zhu2021pure}]{proposition}{propRhoRelation}
\label{prop:rho_relation}
We have $\rho_d^\star (\epsilon) \leq 9 \widetilde \rho_d^\star (\epsilon)$ for any $\epsilon \geq \widetilde \gamma(d)$. Furthermore, if $\widetilde \gamma(d) < \Delta_{\min}/ 2$, $\widetilde \rho_d^\star(0)$ represents the complexity measure for best arm identification with respect to a linear bandit instance with action set $\cX$, target set $\cZ$ and reward function $\widetilde h(x) \ldef \ang*{\psi_d(\theta_\star^d), \psi_d(x)}$.
\end{restatable}

Assuming $\widetilde \gamma(d_\star(\epsilon)) < \min \curly{\epsilon, \Delta_{\min}/2}$, \cref{prop:rho_relation} shows that $\rho_{d_\star(\epsilon)}^\star(\epsilon)$ is at most a constant factor larger than $\widetilde \rho_{d_\star (\epsilon)}^\star(\epsilon)$, which is the $\epsilon$-relaxed complexity measure of a closely related linear bandit problem (without misspecification) in $\R^{d_\star(\epsilon)}$.

\paragraph{Fixed Confidence Setting.} A modified algorithm (and its subroutine, both deferred to \cref{app:alg_misspecification}) is used for the fixed confidence setting with model misspecification. Sample complexity of the modified algorithm is provided as follows.

\begin{restatable}{theorem}{thmDoublingFixedConfidenceMisGen}
\label{thm:doubling_fixed_confidence_mis_gen}
With probability at least $1-\delta$, \cref{alg:doubling_fixed_confidence_mis_gen} starts to output $2 \epsilon$-optimal arms after $N = \widetilde O \paren*{ \log_2(1/\epsilon) \max \curly*{\rho^\star_{d_\star(\epsilon)} (\epsilon), r_{d_\star(\epsilon)}(\zeta)} + 1/\epsilon^2}$ samples, where we hide logarithmic terms besides $\log_2(1/\epsilon)$ in the $\widetilde O$ notation.
\end{restatable}

\begin{remark}
\label{rm:mis_BAI}
The extra $1/\epsilon^2$ term comes from a validation step in the modified algorithm. If the goal is to identify the optimal arm, then this term can be removed with a slight modification of the algorithm. See \cref{app:BAI_misspecification} for detailed discussion.
\end{remark}

\paragraph{Fixed Budget Setting.} Our algorithms for the fixed budget setting are \emph{robust} to model misspecification, and we provide the following guarantees.

\begin{restatable}{theorem}{thmDoublingFixedBudgetMis}
\label{thm:doubling_fixed_budget_mis}
Suppose $\cZ \subseteq \cX$. If $T = \widetilde \Omega \paren*{ \log_2(1/\epsilon) \max \curly*{\rho_{d_\star (\epsilon)}^
    \star(\epsilon), r_{d_\star (\epsilon)}(\zeta)} }$, then \cref{alg:doubling_fixed_budget} 
    outputs an $2\epsilon$-optimal arm with error probability at most
\begin{align*}
    & \log_2 (4/\epsilon)  \abs*{ \cZ }^2  \exp \paren{ - \frac{ T}{ 4096 \, \log_2 (4/\epsilon)  \, \rho_{d_\star(\epsilon)}^\star (\epsilon) } } \\
    & \quad + 2 (\log_2 T)^2 \exp \paren{ - \frac{T}{8 (\log_2 T)^2/ \epsilon^2} }.
\end{align*}
Furthermore, if there exist universal constants such that $\max_{x \in \cX} \norm{\psi_{d_\star(\epsilon)}(x)}^2 \leq c_1$ and $\min_{z \in \cZ} \norm{\psi_{d_\star(\epsilon)}(z_\star) - \psi_{d_\star(\epsilon)}(z)}^2 \geq c_2$, the error probability is upper bounded by
\begin{align*}
    O \Bigg(& \max \curly*{ \log_2(1/\epsilon) \abs*{\cZ}^2, (\log_2 T)^2 }\\
    &\times \exp \paren{ - \frac{c_2 T}{\max \curly*{ \log_2(1/\epsilon), (\log_2 T)^2 } c_1 \rho_{d_\star(\epsilon)}^\star (\epsilon)} } \Bigg).
\end{align*}
\end{restatable}

\section{EXPERIMENT}
\label{sec:experiment}

We empirically compare our \cref{alg:doubling_fixed_confidence} with \rage \citep{fiez2019sequential}, which shares a similar elimination structure to our subroutine (i.e., \cref{alg:subroutine_fixed_confidence}) yet fails to conduct model selection in pure exploration. 
To our knowledge, besides algorithms developed in the present paper, there is no other algorithm that can adapt to the model selection setup for pure exploration linear bandits.\footnote{We defer additional experiment details/results to \cref{app:experiment}.
The purpose of this section is to empirically demonstrate the importance of conducting model selection in pure exploration linear bandits, even on simple problem instances.
We leave large-scale empirical evaluations for future work.}

\paragraph{Problem Instances.}
We conduct experiments with respect to the problem instance used to construct \cref{prop:rho_star_different_d}, which we detail as follows. 

We consider a problem instance with $\cX =\cZ = \curly*{x_i}_{i=1}^{d_\star + 1} \subseteq \R^{d_\star + 1}$ such that $x_i = e_i, \text{ for } i = 1, 2,\dots, d_\star$ and $x_{d_\star+1} = (1-\epsilon) \cdot e_{d_\star} + e_{d_\star+1}$,
where $e_i$ is the $i$-th canonical basis in $\R^{d_\star + 1}$.
The expected reward of each arm is set as $h(x_i) = \ang*{ e_{d_\star},x_i}$, i.e., $\theta_\star = e_{d_\star}$.
One can see that $d_\star$ is the intrinsic dimension and $D = d_\star +1$ is the ambient dimension.
We also notice that $x_\star = x_{d_\star}$ is the best arm with reward $1$, $x_{d_\star+1}$ is the second best arm with reward $1-\epsilon$ and all other arms have reward $0$. The smallest sub-optimality gap is $\epsilon$.
We choose $d_\star = 9$, $D = 10$, and vary $\epsilon$ to control the instance-dependent complexity. By setting $\epsilon$ to be a small value, we create a problem instance such that $\rho^\star_D \gg \rho^\star_{d_\star}$: we have $\rho^\star_{d_\star} = O(d_\star)$ yet $\rho^\star_D = \Omega(1/\epsilon^2)$ (see \cref{app:rho_star_different_d} for proofs).

\begin{table}[H]
  \caption{Comparison of Success Rate}
  \label{tab:success_rate_1}
  \centering
  \begin{tabular}{lcccc}
    \toprule
          $\epsilon $ & $10^{-2}$   & $10^{-3}$ & $10^{-4}$ & $10^{-5}$\\
    \midrule
    \rage    & $100\%$ & $98\%$ & $56\%$ & $62\%$ \\
    Ours    & $100\%$ & $100\%$ & $100\%$ & $100\%$  \\
    \bottomrule
  \end{tabular}
\end{table}
\vspace{-10pt}
\paragraph{Empirical Evaluations.}
We evaluate the performance of each algorithm in terms of success rate, sample complexity and runtime.
We conduct $100$ independent trials for each algorithm. 
Both algorithms are force-stopped after reaching $10$ million samples (denoted as the black line in \cref{fig:comparison_prop}).
We consider an trial as failure if the algorithm fails to identify the best arm within $20$ million samples.
For each algorithm, we calculate the (unverifiable) sample complexity $\tau$ as the smallest integer such that the algorithm (1) empirically identifies the best arm; \emph{and} (2) the algorithm won't change its recommendation for any later rounds $t > \tau$ (up to $20$ million samples).
The (empirical) runtime of the algorithm is calculated as the total time consumed up to round $\tau$.
We average sample complexities and runtimes with respect to succeeded trials.

\begin{figure}[h]
    \centering
    \includegraphics[width=.4\textwidth]{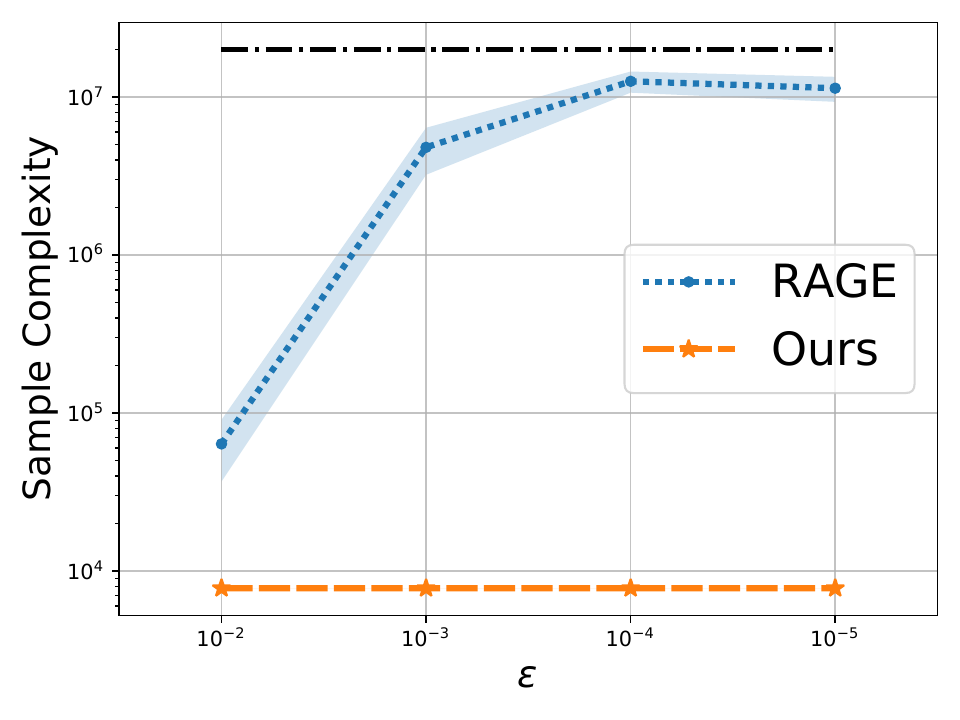}
    \caption{Comparison of Sample Complexity}
    \label{fig:comparison_prop}
\end{figure}

The success rates of \rage and our algorithm are shown in \cref{tab:success_rate_1}. 
The success rate of \rage drops dramatically as $\epsilon$ (the smallest sub-optimality gap) gets smaller. On the other hand, however, our algorithm is not affected by the change of $\epsilon$ since it automatically adapts to the intrinsic dimension $d_\star$: One can immediately see that $h(x_{d_\star}) \geq h(x_{d_\star+1})$ when working in $\R^{d_\star}$. 
Due to the same reason, our algorithm significantly outperforms \rage in sample complexity as well (see \cref{fig:comparison_prop}): Our algorithm adapts to the true sample complexity $\rho^\star_{d_\star}$ yet \rage suffers from complexity $\rho^\star_D \gg \rho^\star_{d_\star}$, especially when $\epsilon$ is small.

The runtime of both algorithms are shown in \cref{tab:runtime}.
Our algorithm is affected by the computational overhead of conducting model selection (e.g., the two dimensional doubling trick). Thus, \rage shows advantages in runtime when $\epsilon$ is relatively large.
However, our algorithm runs faster than \rage when $\epsilon$ gets smaller.
This observation further shows that the implementation overhead can be small in comparison with the sample complexity gains achieved from model selection.

\begin{table}[H]
  \caption{Comparison of Runtime}
  \label{tab:runtime}
  \centering
  \begin{tabular}{lcccc}
    \toprule
          $\epsilon $ & $10^{-2}$   & $10^{-3}$ & $10^{-4}$ & $10^{-5}$\\
    \midrule
    \rage    & $3.46\,$s & $7.87\,$s & $17.33\,$s & $16.81\,$s \\
    Ours    & $12.12\,$s & $11.17\,$s & $12.44\,$s & $12.41\,$s  \\
    \bottomrule
  \end{tabular}
\end{table}
\vspace{-10pt}

It is worth mentioning that simple variations of the problem instance studied in this section have long been considered as hard instances to examine linear bandit pure exploration algorithms \citep{soare2014best, xu2018fully, tao2018best, fiez2019sequential, degenne2020gamification}. Our results show that, both theoretically and empirically, the problem instance becomes quite easy when viewed from the model selection perspective.

\section{DISCUSSION}
\label{sec:discussion}

We initiate the study of model selection in pure exploration linear bandits, in both fixed confidence and fixed budget settings, and design algorithms with near instance optimal guarantees.
Along the way, we develop a novel selection-validation procedure to deal with the understudied fixed budget setting in linear bandits (even without the added challenge of model selection).
We also adapt our algorithms to problems with model misspecification.

We conclude the paper with some directions for future work.
An immediate next step is to conduct large-scale evaluations for model selection in pure exploration linear bandits. 
One may need to develop practical version of our algorithms to bypass the computational overheads of conducting model selection.
Another interesting direction is provide guarantees to general transductive linear bandits, i.e., not restricted to cases $\cZ \subseteq \cX$, in fixed budget setting/misspecified regime. We believe one can use a selection-validation procedure similar to the one developed in \cref{alg:subroutine_fixed_budget}, but with the current validation step replaced by another linear bandit pure exploration algorithm. 
Note that the number of arms to be validated is of logarithmic order.

\subsubsection*{Acknowledgements}
We thank anonymous reviewers for helpful comments. 
This work is partially supported by NSF grant 1934612 and ARMY MURI grant W911NF-15-1-0479.

\bibliographystyle{plainnat}
\bibliography{refs}

\begin{thebibliography}{26}
\providecommand{\natexlab}[1]{#1}
\providecommand{\url}[1]{\texttt{#1}}
\expandafter\ifx\csname urlstyle\endcsname\relax
  \providecommand{\doi}[1]{doi: #1}\else
  \providecommand{\doi}{doi: \begingroup \urlstyle{rm}\Url}\fi

\bibitem[Alieva et~al.(2021)Alieva, Cutkosky, and Das]{alieva2021robust}
Ayya Alieva, Ashok Cutkosky, and Abhimanyu Das.
\newblock Robust pure exploration in linear bandits with limited budget.
\newblock In \emph{International Conference on Machine Learning}, pages
  187--195. PMLR, 2021.

\bibitem[Allen-Zhu et~al.(2020)Allen-Zhu, Li, Singh, and Wang]{allen2020near}
Zeyuan Allen-Zhu, Yuanzhi Li, Aarti Singh, and Yining Wang.
\newblock Near-optimal discrete optimization for experimental design: A regret
  minimization approach.
\newblock \emph{Mathematical Programming}, pages 1--40, 2020.

\bibitem[Audibert et~al.(2010)Audibert, Bubeck, and Munos]{audibert2010best}
Jean-Yves Audibert, S{\'e}bastien Bubeck, and R{\'e}mi Munos.
\newblock Best arm identification in multi-armed bandits.
\newblock In \emph{COLT}, pages 41--53. Citeseer, 2010.

\bibitem[Camilleri et~al.(2021)Camilleri, Katz-Samuels, and
  Jamieson]{camilleri2021high}
Romain Camilleri, Julian Katz-Samuels, and Kevin Jamieson.
\newblock High-dimensional experimental design and kernel bandits.
\newblock \emph{arXiv preprint arXiv:2105.05806}, 2021.

\bibitem[Degenne and Koolen(2019)]{degenne2019pure}
R{\'e}my Degenne and Wouter~M Koolen.
\newblock Pure exploration with multiple correct answers.
\newblock In \emph{Advances in Neural Information Processing Systems}, pages
  14564--14573, 2019.

\bibitem[Degenne et~al.(2020)Degenne, M{\'e}nard, Shang, and
  Valko]{degenne2020gamification}
R{\'e}my Degenne, Pierre M{\'e}nard, Xuedong Shang, and Michal Valko.
\newblock Gamification of pure exploration for linear bandits.
\newblock In \emph{International Conference on Machine Learning}, pages
  2432--2442. PMLR, 2020.

\bibitem[Fiez et~al.(2019)Fiez, Jain, Jamieson, and
  Ratliff]{fiez2019sequential}
Tanner Fiez, Lalit Jain, Kevin~G Jamieson, and Lillian Ratliff.
\newblock Sequential experimental design for transductive linear bandits.
\newblock In \emph{Advances in Neural Information Processing Systems}, pages
  10666--10676, 2019.

\bibitem[Foster et~al.(2019)Foster, Krishnamurthy, and Luo]{foster2019model}
Dylan~J Foster, Akshay Krishnamurthy, and Haipeng Luo.
\newblock Model selection for contextual bandits.
\newblock \emph{arXiv preprint arXiv:1906.00531}, 2019.

\bibitem[Hoffman et~al.(2014)Hoffman, Shahriari, and
  Freitas]{hoffman2014correlation}
Matthew Hoffman, Bobak Shahriari, and Nando Freitas.
\newblock On correlation and budget constraints in model-based bandit
  optimization with application to automatic machine learning.
\newblock In \emph{Artificial Intelligence and Statistics}, pages 365--374.
  PMLR, 2014.

\bibitem[Jaggi(2013)]{jaggi2013revisiting}
Martin Jaggi.
\newblock Revisiting frank-wolfe: Projection-free sparse convex optimization.
\newblock In \emph{International Conference on Machine Learning}, pages
  427--435. PMLR, 2013.

\bibitem[Katz-Samuels and Jamieson(2020)]{katz2020true}
Julian Katz-Samuels and Kevin Jamieson.
\newblock The true sample complexity of identifying good arms.
\newblock In \emph{International Conference on Artificial Intelligence and
  Statistics}, pages 1781--1791. PMLR, 2020.

\bibitem[Katz-Samuels et~al.(2020)Katz-Samuels, Jain, Karnin, and
  Jamieson]{katz2020empirical}
Julian Katz-Samuels, Lalit Jain, Zohar Karnin, and Kevin Jamieson.
\newblock An empirical process approach to the union bound: Practical
  algorithms for combinatorial and linear bandits.
\newblock \emph{arXiv preprint arXiv:2006.11685}, 2020.

\bibitem[Kaufmann et~al.(2016)Kaufmann, Capp{\'e}, and
  Garivier]{kaufmann2016complexity}
Emilie Kaufmann, Olivier Capp{\'e}, and Aur{\'e}lien Garivier.
\newblock On the complexity of best-arm identification in multi-armed bandit
  models.
\newblock \emph{The Journal of Machine Learning Research}, 17\penalty0
  (1):\penalty0 1--42, 2016.

\bibitem[Kiefer and Wolfowitz(1960)]{kiefer1960equivalence}
Jack Kiefer and Jacob Wolfowitz.
\newblock The equivalence of two extremum problems.
\newblock \emph{Canadian Journal of Mathematics}, 12:\penalty0 363--366, 1960.

\bibitem[Lattimore et~al.(2020)Lattimore, Szepesvari, and
  Weisz]{lattimore2020learning}
Tor Lattimore, Csaba Szepesvari, and Gellert Weisz.
\newblock Learning with good feature representations in bandits and in rl with
  a generative model.
\newblock In \emph{International Conference on Machine Learning}, pages
  5662--5670. PMLR, 2020.

\bibitem[Pacchiano et~al.(2020)Pacchiano, Phan, Abbasi-Yadkori, Rao, Zimmert,
  Lattimore, and Szepesvari]{pacchiano2020model}
Aldo Pacchiano, My~Phan, Yasin Abbasi-Yadkori, Anup Rao, Julian Zimmert, Tor
  Lattimore, and Csaba Szepesvari.
\newblock Model selection in contextual stochastic bandit problems.
\newblock \emph{arXiv preprint arXiv:2003.01704}, 2020.

\bibitem[Pukelsheim(2006)]{pukelsheim2006optimal}
Friedrich Pukelsheim.
\newblock \emph{Optimal design of experiments}.
\newblock SIAM, 2006.

\bibitem[Shalev-Shwartz and Ben-David(2014)]{shalev2014understanding}
Shai Shalev-Shwartz and Shai Ben-David.
\newblock \emph{Understanding machine learning: From theory to algorithms}.
\newblock Cambridge university press, 2014.

\bibitem[Soare et~al.(2014)Soare, Lazaric, and Munos]{soare2014best}
Marta Soare, Alessandro Lazaric, and R{\'e}mi Munos.
\newblock Best-arm identification in linear bandits.
\newblock In \emph{Advances in Neural Information Processing Systems}, pages
  828--836, 2014.

\bibitem[Stone(1978)]{stone1978cross}
M~Stone.
\newblock Cross-validation: A review.
\newblock \emph{Statistics: A Journal of Theoretical and Applied Statistics},
  9\penalty0 (1):\penalty0 127--139, 1978.

\bibitem[Stone(1974)]{stone1974cross}
Mervyn Stone.
\newblock Cross-validatory choice and assessment of statistical predictions.
\newblock \emph{Journal of the Royal Statistical Society: Series B
  (Methodological)}, 36\penalty0 (2):\penalty0 111--133, 1974.

\bibitem[Tao et~al.(2018)Tao, Blanco, and Zhou]{tao2018best}
Chao Tao, Sa{\'u}l Blanco, and Yuan Zhou.
\newblock Best arm identification in linear bandits with linear dimension
  dependency.
\newblock In \emph{International Conference on Machine Learning}, pages
  4877--4886, 2018.

\bibitem[Xu et~al.(2018)Xu, Honda, and Sugiyama]{xu2018fully}
Liyuan Xu, Junya Honda, and Masashi Sugiyama.
\newblock A fully adaptive algorithm for pure exploration in linear bandits.
\newblock In \emph{International Conference on Artificial Intelligence and
  Statistics}, pages 843--851, 2018.

\bibitem[Yang and Tan(2021)]{yang2021towards}
Junwen Yang and Vincent~YF Tan.
\newblock Towards minimax optimal best arm identification in linear bandits.
\newblock \emph{arXiv preprint arXiv:2105.13017}, 2021.

\bibitem[Zhu and Nowak(2021)]{zhu2021pareto}
Yinglun Zhu and Robert Nowak.
\newblock Pareto optimal model selection in linear bandits.
\newblock \emph{arXiv preprint arXiv:2102.06593}, 2021.

\bibitem[Zhu et~al.(2021)Zhu, Zhou, Jiang, Gu, Willett, and Nowak]{zhu2021pure}
Yinglun Zhu, Dongruo Zhou, Ruoxi Jiang, Quanquan Gu, Rebecca Willett, and
  Robert Nowak.
\newblock Pure exploration in kernel and neural bandits.
\newblock \emph{arXiv preprint arXiv:2106.12034}, 2021.

\end{thebibliography}


\clearpage
\appendix

\thispagestyle{empty}

\onecolumn \makesupplementtitle

\section{SUPPORTING MATERIALS}
\label{app:supporting}

\subsection{Matrix Inversion and Rounding in Optimal Design}
\label{app:inverse}

Our treatments are similar to the ones discussed in \cite{zhu2021pure}. We provide the details here for completeness.

\textbf{Matrix Inversion.}  The notation $\norm{y}^2_{A_d(\lambda)^{-1}}$ is clear when $A_d(\lambda)$ is invertible. For possibly singular $A_d(\lambda)$, pseudo-inverse is used if $y$ belongs to the range of $A_{d(\lambda)}$; otherwise, we set $\norm{y}_{A_{d}(\lambda)^{-1}}^2 = \infty$. With this (slightly abused) definition of matrix inversion, we discuss how to do rounding next.

\textbf{Rounding in Optimal Design.} For any $\cS \subseteq \cZ$, the following optimal design 
\begin{align*}
    \inf_{\lambda \in \simp_{\cX}} \sup_{y \in \cY(\psi_d(\cS))} \norm{y}^2_{A_{d}(\lambda)^{-1}}
\end{align*}
will select a design $\lambda^\star \in \simp_{\cX}$ such that every $y \in \cY(\psi_d(\cS))$ lies in the range of $A_{d}(\lambda^\star)$.\footnote{If the infimum is not attained, we can simply take a design $\lambda^{\star \star}$ with associated value $\tau^{\star \star} \leq (1+\zeta_0) \inf_{\lambda \in \bLambda_{\cX}} \sup_{\by \in \cY(\bpsi_d(\cS))} \norm{\by}^2_{\bA_{\bpsi_d}(\lambda)^{-1}}$ for a $\zeta_0 > 0$ arbitrarily small. This modification is used in our algorithms as well, and our results (bounds on sample complexity and error probability) goes through with changes only in constant terms. } If $\spn(\cY(\psi_d(\cS))) = \R^d$, then $A_{_d}(\lambda^\star)$ is positive definite (recall that $A_{d}(\lambda^\star) = \sum_{x \in \cX} \lambda_{x} \psi_d(x) \psi_d(x)^\top$ and $\spn(\psi_d(\cX))= \R^d$ comes from the assumption that $\spn(\psi(\cX))= \R^D$). Thus the rounding guarantees in \cite{allen2020near} goes through (Theorem 2.1 therein, which requires a positive definite design; with additional simple modifications dealt as in Appendix B of \cite{fiez2019sequential}).

We now consider the case when $A_{d}(\lambda^\star)$ is singular. Since $\spn(\psi_d(\cX)) = \R^d$, we can always find another $\lambda^\prime$ such that $A_{d}(\lambda^\prime)$ is invertible. For any $\zeta_1 >0$, let $\widetilde \lambda^\star = (1-\zeta_1) \lambda^\star + \zeta_1 \lambda^\prime $. We know that $\widetilde \lambda^\star$ leads to a positive definite design. With respect to $\zeta_1$, we can find another $\zeta_2 > 0$ small enough (e.g., smaller than the smallest eigenvalue of $\zeta_1 A_{d}(\lambda^\prime)$) such that $A_{d}(\widetilde \lambda^\star) \succeq A_{d}((1-\zeta_1) \lambda^\star) + \zeta_2 I$. Since $A_{d}((1-\zeta_1) \lambda^\star) + \zeta_2 I$ is positive definite, for any $y \in \cY(\psi_d(\cS))$, we have 
\begin{align*}
    \norm{y}^2_{A_{d}(\widetilde \lambda^\star)^{-1}} \leq \norm{y}^2_{(A_{d}((1-\zeta_1) \lambda^\star) + \zeta_2 I)^{-1}}.
\end{align*}
Fix any $y \in \cY(\psi_d(\cS))$. Since $y$ lies in the range of $A_{d}(\lambda^\star)$ (by definition of the objective and matrix inversion), we clearly have 
\begin{align*}
    \norm{y}^2_{(A_{d}((1-\zeta_1) \lambda^\star) + \zeta_2 I)^{-1}} 
    \leq \norm{y}^2_{(A_{d}((1-\zeta_1) \lambda^\star))^{-1}}
    \leq \frac{1}{1-\zeta_1} \norm{y}^2_{A_{d}(\lambda^\star)^{-1}}.
\end{align*}
To summarize, we have 
\begin{align*}
    \norm{y}^2_{A_{d}(\widetilde \lambda^\star)^{-1}} \leq \frac{1}{1-\zeta_1} \norm{y}^2_{A_{d}(\lambda^\star)^{-1}},
\end{align*}
where $\zeta_1$ can be chosen arbitrarily small. We can thus send the positive definite design $\widetilde \lambda^\star$ to the rounding procedure in \cite{allen2020near}. We can incorporate the additional $1/(1-\zeta_1)$ overhead, for $\zeta_1 >0$ chosen sufficiently small, into the sample complexity requirement $r_d(\zeta)$ of the rounding procedure.

\subsection{Supporting Theorems and Lemmas}
\label{app:supporting_thm_lm}

\begin{lemma}[\citep{kaufmann2016complexity}]
\label{lm:change_of_measure}
Fixed any pure exploration algorithm $\pi$. Let $\nu$ and $\nu^\prime$ be two bandit instances with $K$ arms such that the distribution $\nu_i$ and $\nu_i^\prime$ are mutually absolutely continuous for all $i \in [K]$. For any almost-surely finite stopping time $\tau$ with respect to the filtration $\curly{\cF_t}_{t \geq 0}$, let $N_i(\tau)$ be the number of pulls on arm $i$ at time $\tau$. We then have 
\begin{align*}
    \sum_{i=1}^K \E_{\nu} [N_i(\tau)] \kl \paren{\nu_i, \nu_i^\prime} \geq \sup_{\cE \in \cF_{\tau}} d \paren{ \P_\nu (\cE), \P_{\nu^\prime} (\cE)},
\end{align*}
where $d(x, y) = x \log(x/y) + (1-x) \log((1-x)/(1-y))$ for $x, y \in [0,1]$ and with the convention that $d(0,0) = d(1,1) = 0$.
\end{lemma}

The following two lemmas largely follow the analysis in \cite{fiez2019sequential}.
\begin{lemma}
\label{lm:rho_stratified_eps}
Let $\cS_k = \curly*{z \in \cZ: \Delta_z < 4 \cdot 2^{-k}}$. We then have 
\begin{align}
    \sup_{k \in [\floor{\log_2 \paren{4/\epsilon}}]} \curly*{  2^{2k}  \iota \paren{ \cY(\psi_d(\cS_k)) } }  \leq 64 \rho_d^\star(\epsilon), \label{eq:psi_stratified_eps_0}
\end{align}
and
\begin{align}
    \sup_{k \in [\floor{\log_2 \paren{4/\epsilon}}]} \curly*{ \max \curly*{ 2^{2k}  \iota \paren{ \cY(\psi_d(\cS_k)) }, r_{d}(\zeta) } } \leq \max \curly{ 64 \rho_d^\star(\epsilon), r_d(\zeta) }, \label{eq:psi_stratified_eps_1}
\end{align}
where $\zeta$ is the rounding parameter.
\end{lemma}
\begin{proof}
For $y = \psi_d(z_\star) - \psi_d(z)$, we define $\Delta_y = \Delta_z = h(z_\star) - h(z)$. We have that 
\begin{align}
    \rho_d^\star (\epsilon) & = \inf_{\lambda \in \simp_{\cX}} \sup_{ y \in \cY^\star(\psi_d(\cZ))} \frac{\norm{y}^2_{A_d(\lambda)^{-1}}}{\max \curly*{\Delta_y, \epsilon}^2} \nonumber\\
    & = \inf_{\lambda \in \simp_{\cX}} \sup_{k \in [\floor{\log_2 \paren{4/\epsilon}}]} \sup_{ y \in \cY^\star(\psi_d(\cS_k))} \frac{\norm{y}^2_{A_d(\lambda)^{-1}}}{\max \curly*{\Delta_y, \epsilon}^2} \nonumber\\
    & \geq \sup_{k \in [\floor{\log_2 \paren{4/\epsilon}}]} \inf_{\lambda \in \simp_{\cX}}  \sup_{ y \in \cY^\star(\psi_d(\cS_k))} \frac{\norm{y}^2_{A_d(\lambda)^{-1}}}{\max \curly*{\Delta_y, \epsilon}^2} \nonumber\\
    & >  \sup_{k \in [\floor{\log_2 \paren{4/\epsilon}}]} \inf_{\lambda \in \simp_{\cX}} \sup_{ y \in \cY^\star(\psi_d(\cS_k))} \frac{\norm{y}^2_{A_d(\lambda)^{-1}}}{\paren{4 \cdot 2^{-k}}^2} \label{eq:psi_stratified_eps_2}\\
    & \geq \sup_{k \in [\floor{\log_2 \paren{4/\epsilon}}]} \inf_{\lambda \in \simp_{\cX}} \sup_{ y \in \cY(\psi_d(\cS_k))} \frac{\norm{y}^2_{A_d(\lambda)^{-1}} / 4}{\paren{4 \cdot 2^{-k}}^2} \label{eq:psi_stratified_eps_3}\\
    & \geq  \sup_{k \in [\floor{\log_2 \paren{4/\epsilon}}]} 2^{2k} \iota (\cY(\psi_d(\cS_k))) / 64, \nonumber
\end{align}
where \cref{eq:psi_stratified_eps_2} comes from the fact that $4 \cdot 2^{-k} \geq \epsilon$ when $k \leq \floor*{\log_2(4/\epsilon)}$; \cref{eq:psi_stratified_eps_3} comes from the fact that $\psi_d(z) - \psi_d(z^\prime) = (\psi_d(z) - \psi_d(z_\star)) + (\psi_d(z_\star) - \psi_d(z^\prime))$. This implies that, for any $k \in [\floor{\log_2 \paren{4/\epsilon}}]$,
\begin{align*}
    \max \curly{2^{2k} \rho(\cY(\psi_d(\cS_k))), r_d(\zeta)} \leq \max \curly{64 \rho_d^\star(\epsilon), r_d(\zeta)}.
\end{align*}
And the desired \cref{eq:psi_stratified_eps_1} immediately follows.
\end{proof}

\begin{lemma}
\label{lm:rho_stratified}
Let $\cS_k = \curly*{z \in \cZ: \Delta_z < 4 \cdot 2^{-k}}$. We then have 
\begin{align}
    \sup_{k \in [\ceil{\log_2 \paren{4/\Delta_{\min}}}]} \curly*{  2^{2k}  \iota \paren{ \cY(\psi_d(\cS_k)) } }  \leq 64 \rho_d^\star, \label{eq:psi_stratified_0}
\end{align}
and
\begin{align}
    \sup_{k \in [\ceil{\log_2 \paren{4/\Delta_{\min}}}]} \curly*{ \max \curly*{ 2^{2k}  \iota \paren{ \cY(\psi_d(\cS_k)) }, r_{d}(\zeta) } } \leq \max \curly{ 64 \rho_d^\star, r_d(\zeta) }, \label{eq:psi_stratified_1}
\end{align}
where $\zeta$ is the rounding parameter.
\end{lemma}
\begin{proof}
Take $\epsilon = \Delta_{\min}$ in \cref{lm:rho_stratified_eps}.
\end{proof}

The following lemma largely follows the analysis in \cite{soare2014best}, with generalization to the transductive setting and more careful analysis in terms of matrix inversion. 
\begin{lemma}
\label{lm:psi_ub_lb}
Fix $\cZ \subseteq \cX \subseteq \R^D$. Suppose $\max_{x \in \cX} \norm{x}^2 \leq c_1$ and $\min_{z \in \cZ \setminus \curly*{z_\star}}\norm{z_\star - z}^2 \geq c_2$ with some absolute constant $c_1$ and $c_2$. We have 
\begin{align*}
    \frac{c_2}{c_1 \Delta_{\min}^2} \leq \rho^\star \ldef \inf_{\lambda \in \simp_{\cX}} \sup_{z \in \cZ \setminus \{z_\star \}} \frac{\norm*{z_{\star}-z}^2_{A(\lambda)^{-1}}}{\Delta_z^2},
\end{align*}
where $\Delta_{\min} = \min_{z \in \cZ \setminus \curly{z_\star}}\curly{\Delta_z}$.
\end{lemma}
\begin{proof}
Let $\lambda^\star$ be the optimal design that attains $\rho^\star$;\footnote{If the infimum is not attained, one can apply the argument that follows with a limit sequence. See footnote in \cref{app:inverse} for more details on how to construct an approximating design.} and let $z^\prime \in \cZ$ be any arm with the smallest sub-optimality gap $\Delta_{\min}$. We then have 
\begin{align}
    \rho^\star & = \max_{z \in \cZ \setminus \{z_\star \}} \frac{\norm*{z_{\star}-z}^2_{A(\lambda^\star)^{-1}}}{\Delta_z^2} \nonumber\\
    & \geq \frac{\norm*{z_{\star}-z^\prime}^2_{A(\lambda^\star)^{-1}}}{\Delta_{z^\prime}^2} \nonumber \\
    & = \frac{\norm*{z_{\star}-z^\prime}^2_{A(\lambda^\star)^{-1}}}{\Delta_{\min}^2}, \label{eq:psi_ub_lb_Delta_min}
\end{align}
where $z_\star - z^\prime$ necessarily lie in the range of $A(\lambda^\star)$ according to the definition of matrix inversion in \cref{app:inverse}. 

We now lower bound $\norm*{z_{\star}-z^\prime}^2_{A(\lambda^\star)^{-1}}$. Note that $A(\lambda^\star)$ is positive semi-definite. We write $A(\lambda^\star) = Q \Sigma Q^\top$ where $Q$ is an orthogonal matrix and $\Sigma$ is a diagonal matrix storing eigenvalues. We assume that the last $k$ eigenvalues of $\Sigma$ are zero. Let $\gamma_{\max} = \norm{A(\lambda^\star)}_2 = \norm{\Sigma}_2$ be the largest eigenvalue, we have $\gamma_{\max} \leq \max_{x \in \cX} \norm{x}^2 \leq c_1$ since $A(\lambda^\star) = \sum_{x \in \cX} \lambda^\star(x) x x^\top$ and $\sum_{x \in \cX} \lambda^\star (x) = 1$. Let $w = Q^\top (z_\star - z^\prime)$. Since $z_\star - z^\prime$ is in the range of $A(\lambda^\star)$, we know that the last $k$ entries of $w$ must be zero. We then have 
\begin{align}
    \norm*{z_{\star}-z^\prime}^2_{A(\lambda^\star)^{-1}} & = (z_\star - z)^\top A(\lambda^\star)^{-1} (z_\star - z) \nonumber \\
    & = w^\top \Sigma^{-1} w  \nonumber \\
    & \geq{\norm{w}^2}/{c_1} \nonumber \\
    & \geq c_2/c_1, \label{eq:psi_ub_lb_dist}
\end{align}
where \cref{eq:psi_ub_lb_dist} comes from fact that $\norm{w}^2 = \norm{z_\star - z^\prime}^2$ and the assumption $\norm{z_\star - z}^2 \geq c_2$ for all $z \in \cZ$.
\end{proof}

\begin{lemma}
\label{lm:relation_log}
The following statements hold.
\begin{enumerate}
    \item $T \geq 4a \log 2a \implies T \geq a \log_2 T$ for $T,a > 0$.
    \item $T \geq 16 a \, (\log 16 a)^2 \implies T \geq a \, (\log_2 T)^2$ for $T,a > 1$.
\end{enumerate}
\end{lemma}

\begin{proof}
We first recall that $T \geq 2a \log a \implies T \geq a \log T$ for $T,a > 0$ \citep{shalev2014understanding}. Since $\log_2 T = \log T / \log 2 < 2 \log T$, the first statement immediately follows.

To prove the second statement, we only need to find conditions on $T$ such that $T \geq 4a \, (\log T)^2$. Note that we have $\sqrt{T} \geq 8 \sqrt{a}  \log 4 \sqrt{a} = 4 \sqrt{a} \log 16 a \implies \sqrt{T} \geq 4 \sqrt{a} \log \sqrt{T} = 2 \sqrt{a} \log T$. For $T,a > 1$, this is equivalent to 
$T \geq 16 a \, (\log 16 a)^2 \implies T \geq 4 a \, (\log T)^2 \geq a \, (\log_2 T)^2$, and thus the second statement follows.
\end{proof}

\subsection{Supporting Algorithms}
\label{app:supporting_alg}

\begin{algorithm}[H]
    \caption{\opt}
    \label{alg:opt}
    \renewcommand{\algorithmicrequire}{\textbf{Input:}}
	\renewcommand{\algorithmicensure}{\textbf{Output:}}
	\begin{algorithmic}[1]
	    \REQUIRE Selection budget $B$, dimension upper bound $D$ and selection function $g(\cdot)$ (which is a function of the dimension $d \in [D]$).
	    \STATE Get $d_k$ such that
	        \begin{align*}
                d_k   = & \max d \\
                & \text{ s.t. } g(d) \leq B, \text{ and } d \in [D] \nonumber .
            \end{align*}
        \ENSURE The selected dimension $d_k$.
	\end{algorithmic}
\end{algorithm}

\section{OMITTED PROOFS FOR SECTION \ref{sec:true_complexity}}

\subsection{Proof of \cref{thm:lower_bound_delta_PAC}}
\label{app:lower_bound_PAC}
\thmLowerBoundDeltaPAC*
\begin{proof}
The proof of the theorem mostly follows the proof of lower bound in \cite{fiez2019sequential}. We additionally consider the model selection problem $(\cX, \cZ$, $\theta_\star \in \Theta_{d_\star})$ and carefully deal with the matrix inversion. 

Consider the instance $(\cX, \cZ$, $\theta_\star \in \Theta_{d_\star})$, where $\cX = \{x_1,\ldots, x_n\}$ and $\spn(\cX) = \R^D$, $\cZ = \{z_1,\ldots, z_m\}$. Suppose that $z_1 = \argmax_{z \in \cZ} \ang*{\theta_\star, z}$. We consider the alternative set $\cC_{d_\star} \ldef \curly*{ \theta \in \Theta_{d_\star}: \exists i \in [m] \text{ s.t. } \ang*{\theta, z_1 - z_i} < 0 }$, where $z_1$ is not the best arm for any $\theta \in \cC_{d_\star}$. Following the ``change of measure'' argument in \cref{lm:change_of_measure}, we know that $\E_{\theta_\star}[\tau] \geq \tau^\star$, where $\tau^\star$ is the solution of the following constrained optimization
\begin{align}
    \tau^\star  & \ldef \min_{t_1, \dots, t_n \in \R_{+}} \sum_{i=1}^n t_i \label{eq:lower_bound_PAC_opt} \\
    & \qquad \text{ s.t. } \inf_{\theta \in \cC_{d_\star}} \sum_{i=1}^n t_i \kl(\nu_{\theta_\star,i}, \nu_{\theta,i}) \geq \log(1/2.4 \delta)\nonumber ,
\end{align}
where we use the notation $\nu_{\theta,i} = \cN ( \ang*{ \theta, x_i }, 1) = \cN (\ang*{\psi_{d_\star}(\theta), \psi_{d_\star}(x_i)}, 1)$ (due to the fact that $\theta \in \cC_{d_\star}$). We also have $\kl(\nu_{\theta_\star,i}, \nu_{\theta,i}) = \frac{1}{2} \ang*{\psi_{d_\star}(\theta_\star)-\psi_{d_\star}(\theta), \psi_{d_\star}(x_i)}^2$.

We next show that for any $t = (t_1, \dots, t_n)^\top \in \R_+^n$ satisfies the constraint of \cref{eq:lower_bound_PAC_opt}, we must have $\psi_{d_\star}(z_1) -\psi_{d_\star}(z_i) \in \spn(\{ \psi_{d_\star}(x_i) : t_i > 0\}), \forall \, 2\leq i \leq m$. Suppose not, there must exists a $\psi_{d_\star}(u) \in \R^{d_\star}$ such that (1) $\ang*{\psi_{d_\star}(u), \psi_{d_\star}(x_i)} = 0$ for all $i \in [n]$ such that $t_i > 0$; and (2) there exists a $2 \leq j \leq m$ such that $\ang*{\psi_{d_\star}(z_1) - \psi_{d_\star}(z_j), \psi_{d_\star}(u) } \neq 0$. Suppose $\ang*{\psi_{d_\star}(z_1 ) - \psi_{d_\star}(z_j), \psi_{d_\star}(u) } > 0$ (the other direction is similar), we can choose a $\theta^\prime \in \Theta_{d_\star}$ such that the first $d_\star$ coordinates of $\theta^\prime$ equals to $\psi_{d_\star}(\theta_\star) - \alpha \, \psi_{d_\star}(u)$ for a $\alpha > 0$ large enough (so that $\theta^\prime \in \cC_{d_\star}$). With such $\theta^\prime$, however, we have 
\begin{align*}
    \sum_{i=1}^n t_i\kl(\nu_{\theta_\star,i}, \nu_{\theta^\prime,i}) = \sum_{i=1}^n t_i \frac{1}{2} \ang*{\alpha \, \psi_{d_\star}(u), \psi_{d_\star}(x_i)}^2 = 0 < \log(1/2.4 \delta),
\end{align*}
which leads to a contradiction. As a result, we can safely calculate $\norm{\psi_{d_\star}(z_1) -\psi_{d_\star}(z_i)}^2_{A_{d_\star}(t)^{-1}}$ or $A_{d_\star}(t)^{-1} (\psi_{d_\star}(z_1) -\psi_{d_\star}(z_i))$ where $A_{d_\star}(t) \ldef \sum_{i=1}^n t_i \psi_{d_\star}(x_i) \psi_{d_\star}(x_i)^\top / \bar t$ and $\bar t \ldef \sum_{i=1}^n t_i$. The rest of the proof follows from the proof of theorem 1 in \cite{fiez2019sequential}.
\end{proof}

\subsection{Proof of \cref{thm:lower_bound_non_interactive}}
\label{app:non_interactive}
\thmLowerBoundNonInteractive*
\begin{proof}
The proof largely follows from the proof of Theorem 3 in \cite{katz2020empirical} (but ignore the $\gamma^\star$ term therein. We are effectively using a weaker lower bound, yet it suffices for our purpose. ). The non-interactive MLE uses at least $\frac{1}{2} \rho^\star_d \log(1/\delta)$ with respect to any feature mapping $\psi_d(\cdot)$ for $d_\star \leq d \leq D$. The statement then follows from the monotonicity of $\curly*{\rho^\star_d}_{d=d_\star}^D$ as shown in \cref{prop:rho_monotonic}.
\end{proof}

\subsection{Proof of \cref{prop:rho_monotonic}}

\propRhoMonotonic*

\begin{proof}

We first prove equivalence results in the general setting in Step 1, 2 and 3; and then apply the results to the model selection problem in Step 4 to prove monotonicity over $\curly*{\rho_d^\star}_{d = d_\star}^D$. 

We consider instance $(\cX, \cZ$, $\theta_\star)$ in the general setting, where $\cX = \{x_1,\ldots, x_n\} \subseteq \R^d$, $\spn(\cX) = \R^d$, $\cZ = \{z_1,\ldots, z_m\}$ and $\theta_\star \in \R^d$. We suppose that $z_1  = \argmax_{z \in \cZ} \ang*{\theta_\star, z}$ is the unique optimal arm and $\spn(\curly*{z_1 - z}_{z \in \cZ \setminus \curly*{z_1}}) = \R^d$. We use the notations $y_j \ldef z_1-z_j$ for $j=2,\ldots, m$, and $\nu_{\theta,i} \ldef \mathcal{N}(x_i^\top \theta, 1)$. For any $t = (t_1, \ldots, t_n)^\top \in \R^n_+$, we also use the notation $A(t) = \sum_{i=1}^n t_i x_i x_i^\top \in \R^{d \times d}$ to denote a design matrix with respect to $t$ ($t$ doesn't need to be inside the simplex $\simp_{\cX}$). We consider any fixed $\delta \in (0,0.15]$.

\textbf{Step 1: Closure of constraints.} Let $\cC$ denote the set of parameters where $z_1$ is no longer the best arm anymore, i.e.,
\begin{align*}
    \cC & \ldef \{\theta \in \R^d : \exists i \in [m] \text{ s.t. } \theta^\top (z_1-z_i) < 0 \}.
\end{align*}
Using the ``change of measure'' argument from \cite{kaufmann2016complexity}, the lower bound is given by the following optimization problem \citep{audibert2010best, fiez2019sequential}
\begin{align*}
    \tau^\star & := \min_{t_1,\ldots, t_n \in \R_+} \sum_{i=1}^n t_i \\
    & \qquad \text{ s.t. } \inf_{\theta \in \cC} \sum_{i=1}^n t_i \kl( \nu_{\theta_\star,i}, \nu_{\theta,i}) \geq \log(1/2.4 \delta). 
\end{align*}
First, we show that the value $\tau^\star$ equals to the value of another optimization problem, i.e.,
\begin{align*}
    \tau^\star & = \min_{t_1,\ldots, t_n \in \R_+} \sum_{i=1}^n t_i \\
    & \qquad \text{ s.t. } \min_{\theta \in \bar{\cC}} \sum_{i=1}^n t_i \kl( \nu_{\theta_\star,i}, \nu_{\theta,i}) \geq \log(1/2.4 \delta),
\end{align*}
where $\bar{\cC} =  \{\theta \in \R^d : \exists i \in [m] \text{ s.t. } \theta^\top (z_1-z_i) \leq 0 \}$. Note that that we must show that the minimum in the constraint is attained, i.e., the $\min_{\theta \in \bar \cC}$ part. We first show the equivalence between the original problem and the problem with respect to $\inf_{\theta \in \bar \cC}$; and then show the equivalence between problems with respect to $\inf_{\theta \in \bar \cC}$ and $\min_{\theta \in \bar \cC}$. We fix any $t=(t_1,\ldots,t_n)^\top \in \R^n_+$.

\textbf{Step 1.1:} We claim that $\inf_{\theta \in \cC} \sum_{i=1}^n t_i \kl( \nu_{\theta_\star,i}, \nu_{\theta,i}) \geq \log(1/2.4 \delta)$ if and only if $\inf_{\theta \in \bar{\cC}} \sum_{i=1}^n t_i \kl( \nu_{\theta_\star,i},  \nu_{\theta,i}) \geq \log(1/2.4 \delta)$. 

Since $\bar{\cC} \supset \cC$, the $\Longleftarrow$ direction is obvious. 

Now, suppose $\inf_{\theta \in \bar{\cC}} \sum_{i=1}^n t_i \kl( \nu_{\theta_\star,i}, \nu_{\theta,i}) < \log(1/2.4 \delta)$. By definition of $\inf$, there exists $\theta_0 \in \bar{\cC}$ such that
\begin{align*}
     \sum_{i=1}^n t_i \kl( \nu_{\theta_\star,i}, \nu_{\theta_0,i}) < \log(1/2.4 \delta).
\end{align*}
Since $\bar \cC$ is the closure of an open set $\cC$, there exists a sequence $\curly*{\theta_j}$ in $\cC$ approaching $\theta_0$. 
Note that 
\begin{align*}
    \sum_{i=1}^n t_i \kl( \nu_{\theta_\star,i}, \nu_{\theta,i}) = \sum_{i=1}^n t_i  \frac{1}{2} (x_i^\top (\theta_\star - \theta))^2 = \frac{1}{2} \norm{\theta_\star-\theta}_{A(t)}^2.
\end{align*}
Then, by the continuity of $\frac{1}{2} \norm{\theta_\star -\theta}_{A(t)}^2$ in $\theta$, there exists a $\theta \in \cC$ such that $\sum_{i=1}^n t_i \kl( \nu_{\theta_\star,i}, \nu_{\theta,i}) < \log(1/2.4 \delta)$. This gives a contradiction and thus proves the $\Longrightarrow$ direction.

\textbf{Step 1.2:} Now, we must show that the infimum is attained whenever $\inf_{\theta \in \bar \cC} \sum_{i=1}^n t_i \kl( \nu_{\theta_\star,i} || \nu_{\theta,i}) \geq \log(1/2.4 \delta)$, that is, there exists $\theta_0 \in \bar \cC $ such that
\begin{align*}
  \sum_{i=1}^n t_i \kl( \nu_{\theta_\star,i}, \nu_{\theta_0,i}) & =   \inf_{\theta \in \bar{\cC}} \sum_{i=1}^n t_i \kl( \nu_{\theta_\star,i}, \nu_{\theta,i}).
\end{align*}

\textbf{Claim:} Fix $t = (t_1,\ldots, t_n)^\top \in \R_+^n$. If $\spn(\{x_i : t_i > 0\}) \neq \R^d$, then $\inf_{\theta \in \bar{\cC}} \sum_{i=1}^n t_i \kl( \nu_{\theta_\star,i} , \nu_{\theta,i}) < \log(1/2.4 \delta)$. 

First, we show the claim. Fix $t = (t_1,\ldots, t_n)^\top \in \R_+^n$ and suppose $\spn(\{x_i : t_i > 0\}) \neq \R^d$. Since $\spn(\{x_i : t_i > 0\}) \neq \R^d$, there exists $u \in \R^d $ such that $u^\top x_i = 0$ for all $i$ such that $t_i > 0$. Since $\{z_1-z_i : i \in [m]\}$ spans $\R^d$ by assumption, there exists $i \in [m]$ such that $u^\top (z_1-z_i) \neq 0$. Suppose that $u^\top (z_1-z_i) < 0$ (the other case is similar). Then, there exists a sufficiently large $\alpha > 0$ such that $(\theta_\star + \alpha u)^\top (z_1 -z_i) < 0$, implying that $\theta_\star + \alpha u \in \cC$. Moreover, by construction of $u$, we have
\begin{align*}
    \sum_{i=1}^n t_i \kl( \nu_{\theta_\star,i}, \nu_{\theta_\star + \alpha u,i}) &  = \sum_{i=1}^n t_i  \frac{1}{2} (x_i^\top (\alpha u))^2 =  \sum_{i : t_i > 0} t_i  \frac{1}{2} (x_i^\top (\alpha u))^2 = 0 < \log(1/2.4 \delta),
\end{align*}
and thus leads to the claim.

Now, suppose $\inf_{\theta \in \bar{\cC}} \sum_{i=1}^n t_i \kl( \nu_{\theta_\star,i}, \nu_{\theta,i}) \geq \log(1/2.4 \delta)$. Then, $\spn(\{x_i : t_i > 0\}) = \R^d$. Then, $\norm{\cdot}_{A(t)}^2$ is a norm, and the set
\begin{align*}
    \curly{ \theta \in \R^d : \frac{1}{2} \norm{\theta-\theta_\star}_{A(t)}^2 \leq \epsilon }
\end{align*}
is compact for every $\epsilon$.  Then, since $\bar{\cC}$ is closed and $\frac{1}{2} \norm{\theta-\theta_\star}_{A(t)}^2$ has compact sublevel sets, there exists a $\theta_0 \in \bar{\cC}$ such that
\begin{align*}
    \sum_{i=1}^n t_i \kl( \nu_{\theta_\star,i} , \nu_{\theta_0,i}) = \inf_{\theta \in \bar{\cC}} \sum_{i=1}^n t_i \kl( \nu_{\theta_\star,i} , \nu_{\theta,i}).
\end{align*}
This shows the equivalence between problems with respect to $\inf_{\theta \in \bar \cC}$ and $\min_{\theta \in \bar \cC}$.

\textbf{Step 2: Rewrite the optimization problem.} Define
\begin{align*}
    \bar{\cC}_i & = \{\theta \in \R^d :  \theta^\top (z_1-z_i) \leq 0 \},
\end{align*}
and note that $\bar{\cC} = \cup_{i=1}^m \bar{\cC}_i$. Observe that
\begin{align*}
        \tau^\star & := \min_{t_1,\ldots, t_n \in \R_+} \sum_{i=1}^n t_i \\
    & \qquad \text{ s.t. } \min_{\theta \in \bar{\cC}} \sum_{i=1}^n t_i \kl( \nu_{\theta_\star,i}, \nu_{\theta,i}) \geq \log(1/2.4 \delta) \\
    & = \min_{t_1,\ldots, t_n \in \R_+} \sum_{i=1}^n t_i \\
    & \qquad \text{ s.t. } \min_{i \in [m]} \min_{\theta \in \bar{\cC}_i} \sum_{i=1}^n t_i \kl( \nu_{\theta_\star,i} , \nu_{\theta,i}) \geq \log(1/2.4 \delta).
\end{align*}

Consider the optimization problem:
\begin{align*}
    \min_{\theta \in \bar{\cC}_i} & \frac{1}{2} \sum_{i=1}^n t_i (x_i^\top (\theta_\star - \theta))^2 = \min_{\theta \in \bar{\cC}_i} \frac{1}{2} \norm{\theta_\star-\theta}_{A(t)}^2 
\end{align*}
Note that since the objective is convex and there exists $\theta \in \R^d$ such that $\theta^\top(z_1-z_i) < 0$, Slater's condition holds and, therefore, strong duality holds. We form the Lagrangian with lagrange multiplier $\gamma \in \R_+$ to obtain
\begin{align*}
    \L(\theta, \gamma ) &= \frac{1}{2} \norm{\theta_\star-\theta}_{A(t)}^2 + \gamma \cdot y_i^\top \theta  \\
\end{align*}
Differentiating with respect to $\theta$ and $\gamma$, we have that (note that $A(t)$ is invertible from the claim in Step 1)
\begin{align*}
\begin{cases}
        \theta  &= \theta_\star - \gamma A(t)^{-1} y_i, \\
     y_i^\top \theta  &= 0.
\end{cases}
\end{align*}
These imply that $\theta_0 \ldef  \theta_\star - \frac{y_i^\top \theta_\star A(t)^{-1} y_i}{y_i^\top A(t)^{-1} y_i}$ and $\gamma_0 \ldef \frac{y_i\top \theta_\star}{y_i^\top A(t)^{-1} y_i} \in \R_+$ satisfy the K.K.T. conditions, and $\theta = \theta_0$ is the minimizer (primal optimal solution) of the constrained optimization problem (note that it's a convex program). Therefore, we have
\begin{align*}
   \min_{\theta \in \bar{\cC}_{i}} & \frac{1}{2} \sum_{i=1}^n t_i (x_i^\top (\theta_\star - \theta))^2 = \frac{(y_i^\top \theta_\star )^2}{\norm{y_i}^2_{A(t)^{-1}}} 
\end{align*}
In conclusion, we have 
\begin{align*}
            \tau^\star & = \min_{t_1,\ldots, t_n \in \R_+} \sum_{i=1}^n t_i \\
    & \qquad \text{ s.t. } \frac{(y_j^\top \theta_\star )^2}{\norm{y_j}^2_{A(t)^{-1}}}  \geq \log(1/2.4 \delta) , \forall \, 2 \leq j \leq m.
\end{align*}

\textbf{Step 3: Re-express the optimization problem.} Furthermore, we have that
\begin{align}
    \tau^\star & = \min_{s, t_1,\ldots, t_n \in \R_+} s \label{eq:monotonic_original_opt}\\
    & \qquad \text{ s.t. }  (y_j^\top \theta_\star )^2 \geq \log(1/2.4 \delta) \norm{y_j}^2_{A(t)^{-1}}  , \forall \, 2 \leq j \leq m \nonumber\\
    & \qquad \qquad s \geq \sum_{i=1}^n t_i. \nonumber
\end{align}
Rearranging these constraints, we have that 
\begin{align*}
    s \geq \sum_{i=1}^n t_i \geq \log(1/2.4 \delta) \sum_{i=1}^n t_i \frac{\norm{y_j}^2_{A(t)^{-1}}}{(y_j^\top \theta_\star )^2} = \log(1/2.4 \delta)  \frac{\norm{y_j}^2_{A(\lambda)^{-1}}}{(y_j^\top \theta_\star )^2} , \forall \, 2 \leq j \leq m.
\end{align*}
We do a change of variables $\lambda \in \simp_{\cX}$ and $\lambda_i = \frac{t_i}{\sum_{i=1}^n t_i}$, and the optimization problem is equivalent to
\begin{align*}
     \tau^\star & = \min_{s \in \R_+, \lambda \in \simp_{\cX}} s \\
    & \qquad \text{ s.t. }   s \geq \max_{j =2,\ldots, m} \log(1/2.4 \delta)  \frac{\norm{y_j}^2_{A(\lambda)^{-1}}}{(y_j^\top \theta_\star )^2}.
\end{align*}
Thus, we have that
\begin{align*}
    \tau^\star \geq \inf_{\lambda \in \simp_{\cX}} \max_{j=2,\ldots, m} \frac{\norm{y_j}^2_{A(\lambda)^{-1}}}{(y_j^\top \theta_\star )^2} \log(1/2.4 \delta).
\end{align*}
Now let 
\begin{align*}
    \widetilde \tau^\star \ldef  \inf_{\lambda \in \simp_{\cX}} \max_{j=2,\ldots, m} \frac{\norm{y_j}^2_{A(\lambda)^{-1}}}{(y_j^\top \theta_\star )^2} \log(1/2.4 \delta) = \max_{j=2,\ldots, m} \frac{\norm{y_j}^2_{A(\lambda^\star)^{-1}}}{(y_j^\top \theta_\star )^2} \log(1/2.4 \delta),
\end{align*}
where $\lambda^\star$ is the optimal design of the above optimization problem.\footnote{Again, if the infimum is not attained, one can apply the argument that follows with a limit sequence. See footnote in \cref{app:inverse} for more details on how to construct an approximating design.} Set $\widetilde t = \widetilde \tau^\star \lambda^\star \in \R_+^n$ with $\widetilde t_i = \widetilde \tau^\star \lambda^\star_i \in \R_+$, we can then see that 
\begin{align*}
    \sum_{i=1}^n \widetilde t_i = \widetilde \tau^\star  = \max_{j=2,\ldots, m} \sum_{i=1}^n \widetilde t_i \frac{\norm{y_j}^2_{A(\widetilde t)^{-1}}}{(y_j^\top \theta_\star )^2} \log(1/2.4 \delta)  , \forall \, 2 \leq j \leq m.
\end{align*}
and such $\curly*{\widetilde{t}_i}$ satisfies the constraints in the original optimization problem described in \cref{eq:monotonic_original_opt}. 
As a result, we have $\tau^\star \leq \widetilde \tau^\star$.

We now can write 
\begin{align}
    \tau^\star = \inf_{\lambda \in \simp_{\cX}} \max_{j=2,\ldots, m} \frac{\norm{y_j}^2_{A(\lambda)^{-1}}}{(y_j^\top \theta_\star )^2} \log(1/2.4 \delta) = \rho^\star \log(1/2.4 \delta). \label{eq:monotonic_equivalence}
\end{align}

\textbf{Step 4: Monotonicity.} We now apply the established equivalence to the model selection problem and prove monotonicity over $\curly*{\rho_d^\star}_{d= d_\star}^D$.

Now, define
\begin{align*}
    \tau^\star_{d_\ell} & = \min_{t_1,\ldots, t_n \in \R_+} \sum_{i=1}^n t_i \\
    & \qquad \text{ s.t. } \inf_{\theta \in \cC_{d_\ell}} \sum_{i=1}^n t_i \kl( \nu_{\theta_\star,i}, \nu_{\theta,i}) \geq \log(1/2.4 \delta),
\end{align*}
where $\cC_{d_\ell} = \curly*{ \theta \in \R^D : \forall j > d_\ell : \theta_j = 0 \wedge \exists i \in [m] \text{ s.t. } \theta^\top (z_1-z_i) < 0 }$. Let $d_\star \leq d_1 \leq d_2 \leq D$. Then, since the optimization problem in $\tau^\star_{d_1}$ has fewer constraints than the optimization problem in $\tau^\star_{d_2}$, we have that $\tau^\star_{d_1} \leq \tau^\star_{d_2}$. The established equivalence in \cref{eq:monotonic_equivalence} can be applied with respect to feature mappings $\psi_d(\cdot)$ for $d_\star \leq d \leq D$ (note that we necessarily have $\spn(\curly*{\psi_d(z_\star) - \psi_d(z)}_{z \in \cZ \setminus \curly*{z_\star}}) = \R^d$ as long as $\spn(\curly*{z_\star - z}_{z \in \cZ \setminus \curly*{z_\star}}) = \R^D$). Therefore, we have
\begin{align*}
 \rho^\star_{d_1} \log(1/2.4 \delta)  =  \tau^\star_{d_1}  \leq \tau^\star_{d_2} = \rho^\star_{d_2} \log(1/2.4 \delta),
\end{align*}
leading to the desired result.
\end{proof}

\subsection{Proof of \cref{prop:rho_star_different_d}}
\label{app:rho_star_different_d}

\propRhoStarDifferentD*

\begin{proof}
For any $\lambda \in \simp_{\cX}$, we define
\begin{align*}
    \rho_{d}(\lambda) \ldef \max_{z \in \cZ \setminus \curly*{z_\star }} \frac{\norm*{\psi_d(z_{\star})-\psi_d(z)}^2_{A_{d}(\lambda)^{-1}}}{(h(z_\star) - h(z))^2},
\end{align*}
and 
\begin{align*}
    \iota_{d}(\lambda) \ldef \max_{z \in \cZ \setminus \curly*{z_\star }} {\norm*{\psi_d(z_{\star})-\psi_d(z)}^2_{A_{d}(\lambda)^{-1}}}.
\end{align*}

We consider an instance $\cX =\cZ = \curly*{x_i}_{i=1}^{d_\star + 1} \subseteq \R^{d_\star + 1}$ and expected reward function $h(\cdot)$. The action set is constructed as follows:
\begin{align*}
    x_i = e_i, \text{ for } i = 1, 2,\dots, d_\star, \quad x_{d_\star+1} = (1-\epsilon) \cdot e_{d_\star} + e_{d_\star+1},
\end{align*}
where $e_i$ is the $i$-th canonical basis in $\R^{d_\star + 1}$. The expected reward of each action is set as 
\begin{align*}
    h(x_i) \ldef \ang*{x_i, e_{d_\star}}.
\end{align*}
One can easily see that $d_\star$ is the intrinsic dimension of the problem (in fact, it is the smallest dimension such that linearity in rewards is preserved).

We notice that $\theta_\star \in \R^{d_\star}$; $x_\star = x_{d_\star}$ is the best arm with reward $1$, $x_{d_\star+1}$ is the second best arm with reward $1-\epsilon$ and all other arms have reward $0$. The smallest sub-optimality gap is $\Delta_{\min} = \epsilon$. $\epsilon \in (0,1/2]$ is selected such that $1/4\epsilon^2 > 2d_\star + \gamma$ for any given $\gamma > 0$.\footnote{One can also add an additional arm $x_{0} = e_D/2$ so that $\spn(\curly*{x_\star - x}_{x \in \cX}) = \R^{d_\star + 1}$ (the lower bound on $\rho_{d_\star + 1}^\star$ will be changed to $1/16\epsilon^2$).}

We first consider truncating arms into $\R^{d_\star}$. For any $\lambda \in \simp_{\cX}$, we notice that $A_{d_\star}(\lambda) = \sum_{x \in \cX} \lambda_x \psi_{d_\star}(x) \psi_{d_\star}(x)^\top$ is a diagonal matrix with the $d_\star$-th entry being $\lambda_{x_{d_\star}} + (1-\epsilon)^2 \lambda_{x_{d_\star+1}}$ and the rest entries being $\lambda_{x_i}$. We first show that $\iota_{d_\star}^\star \geq d_{\star}-1$ by contradiction as follows. Suppose $\iota^\star_{d_\star} < d_\star-1$. Since $\norm*{\psi_{d_\star}(x_\star) - \psi_{d_\star}(x_i)}^2_{A_{d_\star}(\lambda)^{-1}} \geq 1/\lambda_{x_i}$ for $i = 1, 2, \dots, d_\star-1$, we must have $\lambda_{x_i} > 1/(d_\star-1)$ for $i = 1, 2, \dots, d_\star-1$. Thus, $\sum_{i=1}^{d_\star-1} \lambda_{x_i} >  1$, which leads to a contradiction for $\lambda \in \simp_{\cX}$. We next analyze $\rho^\star_d$. Let $\lambda^\prime \in \simp_{\cX}$ be the design such that $\lambda^\prime_{x_i} = 1/d_{\star}$ for $i = 1,\dots, d_\star$. With design $\lambda^\prime$, we have $\norm*{\psi_{d_\star}(x_\star) - \psi_{d_\star}(x_i)}^2_{A_{d_\star}(\lambda^\prime)^{-1}} = 2d_\star$ for $i = 1, 2, \dots, d_\star-1$ and $\norm*{\psi_{d_\star}(x_\star) - \psi_{d_\star}(x_{d_\star + 1})}^2_{A_{d_\star}(\lambda^\prime)^{-1}} = \epsilon^2 d_\star$. As a result, we have $\rho_{d_\star}(\lambda^\prime) \leq 2d_\star$, and thus $\rho^\star_{d_\star} \leq \rho_{d_\star}(\lambda^\prime) \leq 2d_\star$.

We now consider arms in the original space, i.e., $\R^{d_\star+1}$. We first upper bound $\iota^\star_{d_\star + 1}$. With an uniform design $\lambda^{\prime \prime}$ such that $\lambda_{x_i}^{\prime \prime} = 1/(d_\star+1), \forall i \in [d_\star + 1]$, we have $\iota^\star_{d_\star+1} \leq \iota_{d_\star + 1}(\lambda^{\prime \prime}) \leq \max \curly*{(3-\epsilon)/(2-\epsilon), \epsilon^2/(2-\epsilon) + 1} \cdot (d_\star + 1) \leq 5(d_\star+1)/3$ when $\epsilon \in (0,1/2]$. In fact, with the same design, we can also upper bound $\iota(\cY(\psi_{d_\star+1}(\cX))) \leq 3(d_\star + 1)$. We analyze $\rho_{d_\star + 1}^\star$ now. Since $\max_{x \in \cX} \norm*{x}^2 \leq 4$ and $\min_{x \in \cX \setminus \curly*{x_\star}} \norm*{x_\star - x}^2 \geq 1$, \cref{lm:psi_ub_lb} leads to the fact that $\rho^\star_{d_\star+1} \geq 1/4 \epsilon^2$. Note that we only have $\min_{x \in \cX \setminus \curly*{x_\star}} \norm*{\psi_{d_\star}(x_\star) - \psi_{d_\star }(x)}^2 \geq \epsilon^2$ when truncating arms into $\R^{d_\star}$.

To summarize, for any given $\gamma > 0$, we have $\rho^\star_{d_\star+1} > \rho^\star_{d_\star} + \gamma$ yet $\iota^\star_{d_\star+1} \leq 2\iota^\star_{d_\star}$ (when $d_\star \geq 11$). Further more, we also have $\iota(\cY(\psi_{d_\star+1}(\cX))) \leq 4 \iota(\cY(\psi_{d_\star}(\cX)))$ (when $d_\star \geq 7$) since $\iota(\cY(\psi_{d_\star}(\cX))) \leq \iota_{d_\star}^\star$.
\end{proof}

\section{OMITTED PROOFS FOR SECTION \ref{sec:fixed_confidence}}
\label{app:fixed_confidence}

\subsection{Proof of \cref{lm:subroutine_fixed_confidence}}

\lmSubroutineFixedConfidence*

\begin{proof}
We consider event 
\begin{align*}
    \cE_k = \curly*{z_\star \in \widehat \cS_k \subseteq \cS_k},
\end{align*}
and prove through induction that 
\begin{align*}
    \P \paren{ \cE_{k+1} \mid \cap_{i \leq k} \cE_i } \geq 1 - \delta_{k},
\end{align*}
where $\delta_0 \ldef 0$. Recall that $\cS_k = \curly*{z \in \cZ: \Delta_z < 4 \cdot 2^{-k}}$ (with $\cS_1 = \cZ$).

\textbf{Step 1: The induction.} We have $\curly*{ z_\star \in \widehat \cS_1 \subseteq \cS_1 }$ since $\widehat \cS_1 = \cS_1 = \cZ$ by definition for the base case (recall that we assume $\max_{z \in \cZ} \Delta_z \leq 2$). We now assume that $\cap_{i \leq k}\cE_{i}$ holds true and we prove for iteration $k+1$. We only need to consider the case when $\abs*{\widehat \cS_k} > 1$, which implies $\abs*{\cS_k} > 1$ and thus $k \leq \floor*{ \log_2 (4 / \Delta_{\min}) }$.

\textbf{Step 1.1: $d_k \geq d_\star$ (Linearity is preserved).} Since $\widehat \cS_{k} \subseteq \cS_k$, we have 
\begin{align}
    g_k(d_\star) & = \max \curly*{ {2^{2k} \iota(\cY(\psi_{d_\star}(\widehat \cS_k)))}, r_{d_\star}(\zeta) } \nonumber \\
    & \leq \max \curly*{ {2^{2k} \iota(\cY(\psi_{d_\star}(\cS_k)))}, r_{d_\star}(\zeta) } \nonumber\\
    & \leq \max \curly*{64 \rho_{d_\star}^\star, r_{d_\star}(\zeta)} \label{eq:subroutine_fixed_confidence_rho_stratified}\\
    & \leq B \label{eq:subroutine_fixed_confidence_B},
\end{align}
where \cref{eq:subroutine_fixed_confidence_rho_stratified} comes from \cref{lm:rho_stratified} and \cref{eq:subroutine_fixed_confidence_B} comes from the assumption. As a result, we know that $d_k \geq d_\star$ since $d_k$ is selected as the largest integer such that $g_k(d_k) \leq B$. 

\textbf{Step 1.2: Concentration.} Let $\curly*{x_1, \ldots, x_{N_k}} $ be the arms pulled at iteration $k$ and $\curly*{r_1, \ldots, r_{N_k}}$ be the corresponding rewards. Let $\widehat{\theta}_k = A_k^{-1} b_k \in \R^{d_k}$ where $A_k = \sum_{i=1}^{N_k} \psi_{d_k}(x_i) \psi_{d_k}(x_i)^\top$, and $b_k =  \sum_{i=1}^{N_k} \psi_{d_k}(x_i) b_i$. Since $d_k \geq d_\star$ and the model is well-specified, we can write $r_i = \ang*{\theta_\star, x_i} + \xi_i = \ang*{\psi_{d_k}(\theta_\star), \psi_{d_k}(x_i)} + \xi_i$, where $\xi_i$ is i.i.d. generated $1$-sub-Gaussian noise. For any $y \in \cY(\psi_{d_k}(\widehat \cS_k))$, we have
\begin{align*}
    \ang{y, \widehat \theta_k - \psi_{d_k}(\theta_\star)} & = y^\top A_k^{-1} \sum_{i=1}^{N_k} \psi_{d_k}(x_i) r_i - y^\top \psi_{d_k}(\theta_\star) \nonumber \\
    & = y^\top A_k^{-1} \sum_{i=1}^{N_k} \psi_{d_k}(x_i) \paren{ \psi_{d_k}(x_i) ^\top \psi_{d_k}(\theta_\star) + \xi_i} - y^\top \psi_{d_k}(\theta_\star) \nonumber \\
    & = y^\top A_k^{-1} \sum_{i=1}^{N_k} \psi_{d_k}(x_i) \xi_i.
\end{align*}
Since $\xi_i$s are independent 1-sub-Gaussian random variables, we know that the random variable $ y^\top A_k^{-1} \sum_{i=1}^{N_k} \psi_{d_k}(x_i) \xi_i $ has variance proxy $\sqrt{\sum_{i=1}^{N_k} \paren*{y^\top A_k^{-1} \sum_{i=1}^{N_k} \psi_{d_k}(x_i)}^2 } = \norm{y}_{A_k^{-1}}$. Combining the standard Hoeffding's inequality with a union bound leads to
\begin{align}
    \P \paren{ \forall y \in \cY(\psi_{d_k}(\widehat \cS_k)) , \abs{ \ang{ y, \widehat \theta_k - \psi_{d_k}(\theta_\star)  } } \leq \norm{y}_{A_k^{-1}} \sqrt{2 \log \paren{ {\abs*{\widehat \cS_k}^2}/{\delta_k} }}  } \geq 1-\delta_k, \label{eq:subroutine_fixed_confidence_event}
\end{align}
where we use the fact that $\abs*{ \cY(\psi_{d_k}(\widehat \cS_k)) } \leq \abs*{\widehat \cS_k}^2/2$ in the union bound. 

\textbf{Step 1.3: Correctness.} We prove $z_\star \in \widehat \cS_{k+1} \subseteq \cS_{k+1}$ under the good event analyzed in \cref{eq:subroutine_fixed_confidence_event}.

\textbf{Step 1.3.1: $z_\star \in \widehat \cS_{k+1}$.} For any $\widehat z \in \widehat \cS_k$ such that $\widehat z \neq z_\star$, we have 
\begin{align*}
    \ang*{ \psi_{d_k}(\widehat z) - \psi_{d_k}(z_\star) , \widehat \theta_k } & \leq \ang{ \psi_{d_k}(\widehat z) - \psi_{d_k}(z_\star) , \psi_{d_k} (\theta_\star) } + \norm{\psi_{d_k}(\widehat z) - \psi_{d_k}(z_\star)}_{A_k^{-1}} \sqrt{2 \log \paren{ {\abs*{\widehat \cS_k}^2}/{\delta_k} }} \\
    & = h(\widehat z ) - h(z_\star) + \norm{\psi_{d_k}(\widehat z) - \psi_{d_k}(z_\star)}_{A_k^{-1}} \sqrt{2 \log \paren{ {\abs*{\widehat \cS_k}^2}/{\delta_k} }} \\
    & < \norm{\psi_{d_k}(\widehat z) - \psi_{d_k}(z_\star)}_{A_k^{-1}} \sqrt{2 \log \paren{ {\abs*{\widehat \cS_k}^2}/{\delta_k} }}.
\end{align*}
As a result, $z_\star$ remains in $\widehat \cS_{k+1}$ according to the elimination criteria.

\textbf{Step 1.3.2: $\widehat \cS_{k+1} \subseteq \cS_{k+1}$.} Consider any $z \in \widehat \cS_k \cap \cS_{k+1}^c$, we know that $\Delta_z \geq 2\cdot 2^{-k}$ by definition. Since $z_\star \in \widehat \cS_k$, we then have 
\begin{align}
    \ang*{ \psi_{d_k}(z_\star) - \psi_{d_k}(z) , \widehat \theta_k } & \geq \ang{ \psi_{d_k}(z_\star) - \psi_{d_k}(z) , \psi_{d_k} (\theta_\star) } - \norm{\psi_{d_k}(z_\star) - \psi_{d_k}(z)}_{A_k^{-1}} \sqrt{2 \log \paren{ {\abs*{\widehat \cS_k}^2}/{\delta_k} }} \nonumber \\
    & = h(z_\star) - h(z) - \norm{\psi_{d_k}(z_\star) - \psi_{d_k}(z)}_{A_k^{-1}} \sqrt{2 \log \paren{ {\abs*{\widehat \cS_k}^2}/{\delta_k} }} \nonumber \\ 
    & \geq 2 \cdot 2^{-k} - \norm{\psi_{d_k}(z_\star) - \psi_{d_k}(z)}_{A_k^{-1}} \sqrt{2 \log \paren{ {\abs*{\widehat \cS_k}^2}/{\delta_k} }} \nonumber \\
    & \geq \norm{\psi_{d_k}(z_\star) - \psi_{d_k}(z)}_{A_k^{-1}} \sqrt{2 \log \paren{ {\abs*{\widehat \cS_k}^2}/{\delta_k} }} \label{eq:subroutine_fixed_confidence_rounding},
\end{align}
where \cref{eq:subroutine_fixed_confidence_rounding} comes from the fact that $\norm{\psi_{d_k}(z_\star) - \psi_{d_k}(z)}_{A_k^{-1}} \sqrt{2 \log \paren{ {\abs*{\widehat \cS_k}^2}/{\delta_k} }} \leq 2^{-k}$, which is resulted from the choice of $N_k$ and the guarantee in \cref{eq:rounding} from the rounding procedure. As a result, we have $z \notin \widehat \cS_{k+1}$ and $\widehat \cS_{k+1} \subseteq \cS_{k+1}$.

To summarize, we prove the induction at iteration $k+1$, i.e.,
\begin{align*}
    \P \paren{ \cE_{k+1} \mid \cap_{i< k+1} \cE_{i} } \geq 1 - \delta_k.
\end{align*}

\textbf{Step 2: The error probability.} Let $\cE = \cap_{i=1}^{n+1} \cE_{i}$ denote the good event, we then have 
\begin{align}
    \P \paren{ \cE } & = \prod_{k=1}^{n} \P \paren{ \cE_k \mid \cE_{k-1} \cap \dots \cap \cE_1 } \nonumber \\
    & = \prod_{k=1}^{n} \paren{1 - \delta_k} \nonumber \\
    & \geq \prod_{k=1}^\infty \paren{1 - \delta/k^2} \nonumber \\
    & = \frac{\sin(\pi \delta)} {\pi \delta} \nonumber \\
    & \geq 1- \delta \label{eq:subroutine_fixed_confidence_delta},
\end{align}
where we use the fact that ${\sin(\pi \delta)}/{\pi \delta} \geq 1-\delta$ for any $\delta \in (0,1)$ in \cref{eq:subroutine_fixed_confidence_delta}. 
\end{proof}

\subsection{Proof of \cref{thm:doubling_fixed_confidence}}

\thmDoublingFixedConfidence*

\begin{proof}
The proof is decomposed into three steps: (1) locating good subroutines; (2) bounding error probability and (3) bounding unverifiable sample complexity.

\textbf{Step 1: Locating good subroutines.} Consider $B_\star = \max \curly*{64 \rho^\star_{d_\star}, r_{d_\star}(\zeta)}$ and $n_\star = \ceil*{ \log_2 (2/\Delta_{\min}) }$. For any subroutines invoked with $B_i \geq B_\star$ and $n_i \geq n_\star$, we know that, from \cref{lm:subroutine_fixed_confidence}, the output set of arms are those with sub-optimality gap $< \Delta_{\min}$, which is a singleton set containing the optimal arm, i.e., $\curly*{z_\star}$. Let $i_\star = \ceil*{\log_2 (B_\star)}$, $j_\star = \ceil*{ \log_2 (n_\star) }$ and $\ell_\star = i_\star + j_\star$. We know that in outer loops $\ell \geq \ell_\star$, there must exists at least one subroutine invoked with $B_i = 2^{i_\star} \geq B_\star$ and $n_i = 2^{j_\star} \geq n_\star$. Once a subroutine, invoked with $B_i \geq B_\star$, outputs a singleton set, it must be the optimal arm $z_\star$ according to \cref{lm:subroutine_fixed_confidence} (up to small error probability, analyzed as below). Since, within each outer loop $\ell$, the value of $B_i = 2^{\ell - i}$ is chosen in a decreasing order, updating the recommendation and breaking the inner loop once a singleton set is identified will not miss the chance of recommending the optimal arm in later subroutines within outer loop $\ell$.

\textbf{Step 2: Error probability.} We consider the good event where all subroutines invoked in \cref{alg:doubling_fixed_confidence} with $B_i \geq B_\star$ and (any) $n_i$ correctly output a set of arms with sub-optimality gap $<2^{1-n_i}$ with probability at least $1 - \delta_\ell$, as shown in \cref{lm:subroutine_fixed_confidence}. This good event clearly happens with probability at least $1 - \sum_{\ell = 1}^\infty \sum_{i=1}^\ell \delta_\ell = 1- \sum_{\ell = 1}^\infty \delta/(2 \ell^2) > 1 - \delta$, after applying a union bound argument. We upper bound the unverifiable sample complexity under this event in the following.

\textbf{Step 3: Unverifiable sample complexity.} 
For any subroutine invoked within outer loop $\ell \leq \ell_\star$, we know, from \cref{alg:subroutine_fixed_budget}, that its sample complexity is upper bounded by (note that $\abs*{\cZ}^2 \geq 4$ trivially holds true)
\begin{align*}
    N_\ell & \leq n_i \paren{ B_i \cdot \paren{2.5 \, \log(\abs{\cZ}^2/\delta_{\ell_\star})} + 1 } \\
    & \leq \gamma_\ell \, 3.5 \, \log \paren{ 2 \abs{\cZ}^2 \ell_\star^3/\delta }.
\end{align*}
Thus, the total sample complexity up to the end of outer loop $\ell_\star$ is upper bounded by 
\begin{align*}
    N & \leq \sum_{\ell = 1}^{\ell_\star} \ell N_\ell \\
    & \leq 3.5\, \log \paren{ 2 \abs{\cZ}^2 \ell_\star^3/\delta } \sum_{\ell = 1}^{\ell_\star} \ell 2^\ell \\
    & \leq 7 \, \log \paren{ 2 \abs{\cZ}^2 \ell_\star^3/\delta } \ell_\star 2^{\ell_\star}.
\end{align*}

Recall that $\tau_\star = \log_2(4/\Delta_{\min}) \max \curly*{\rho^\star_{d_\star}, r_{d_\star}(\zeta)}$. By definition of $\ell_\star$, we have
\begin{align*}
    \ell_\star \leq \log_2 \paren{ 4  \log_2(4/\Delta_{\min})  \max \curly*{64 \rho^\star_{d_\star}, r_{d_\star}(\zeta)} } = O(\log_2 (\tau_\star)),
\end{align*}
and
\begin{align*}
    2^{\ell_\star} & = 2^{(i_\star + j_\star)} \\
    & \leq 4  \paren{ \log_2(2/\Delta_{\min})+1}  \max \curly*{64 \rho^\star_{d_\star}, r_{d_\star}(\zeta)}, \\
    & = 4  \log_2(4/\Delta_{\min})  \max \curly*{64 \rho^\star_{d_\star}, r_{d_\star}(\zeta)}, \\
    & = O(\tau_\star).
\end{align*}

 The unverifiable sample complexity is thus upper bounded by
\begin{align*}
    N & \leq 1792 \, \tau_\star \cdot \paren{\log_2(\tau_\star) + 8} \cdot \log \paren{ {2 \abs{\cZ}^2 \paren{\log_2(\tau_\star) + 8}^3}/{\delta} } \\
    & = O \paren{ \tau_\star \log_2(\tau_\star) \log(\abs{\cZ} \log_2(\tau_\star)/\delta) }.
\end{align*}
\end{proof}

\section{OMITTED PROOFS FOR SECTION \ref{sec:fixed_budget}}

\subsection{Proof of \cref{lm:subroutine_fixed_budget}}

\lmSubroutineFixedBudget*

\begin{proof}
We consider event 
\begin{align*}
    \cE_k = \curly*{z_\star \in \widehat \cS_k \subseteq \cS_k},
\end{align*}
and prove through induction that 
\begin{align*}
    \P \paren{ \cE_{k+1} \mid \cap_{i \leq k} \cE_i } \geq 1 - \delta_{k},
\end{align*}
where the value of $\curly*{\delta_k}_{k=0}^{n}$ will be specified in the proof. 

\textbf{Step 1: The induction.} The base case $\curly*{z_\star \in \widehat \cS_1 \subseteq \cS_1}$ holds with probability $1$ by construction (thus, we have $\delta_0 = 0$). Conditioned on events $\cap_{i=1}^k \cE_i$, we next analyze the event $\cE_{k+1}$. We only need to consider the case when $\abs*{\widehat \cS_k} > 1$, which implies $\abs*{\cS_k} > 1$ and thus $k \leq \floor*{ \log_2 (4 / \Delta_{\min}) }$.

\textbf{Step 1.1: $d_k \geq d_\star$ (Linearity is preserved).} We first notice that $\widetilde D$ is selected as the largest integer such that $r_{\widetilde D}(\zeta) \leq T^\prime$, where $r_d(\zeta)$ represents the number of samples needed for the rounding procedure in $\R^d$ (with parameter $\zeta$). When $T/ n \geq r_{d_\star}(\zeta) + 1$, we have $\widetilde D \geq d_\star$ since $T^\prime \geq T/ n -1 \geq r_{d_\star}(\zeta)$. 
Thus, for whatever $d_k \in [\widetilde D]$ selected, we always have $r_{d_k}(\zeta) \leq r_{\widetilde D}(\zeta) \leq T^\prime$ and can thus safely apply the rounding procedure described in \cref{eq:rounding}.

Since $\widehat \cS_{k} \subseteq \cS_k$, we also have 
\begin{align}
    g_k(d_\star) & =  {2^{2k} \iota(\cY(\psi_{d_\star}(\widehat \cS_k)))} \nonumber \\
    & \leq {2^{2k} \iota(\cY(\psi_{d_\star}(\cS_k)))} \nonumber\\
    & \leq 64 \rho_{d_\star}^\star \label{eq:subroutine_fixed_budget_rho_stratified}\\
    & \leq B \label{eq:subroutine_fixed_budget_B},
\end{align}
where \cref{eq:subroutine_fixed_budget_rho_stratified} comes from \cref{lm:rho_stratified} and \cref{eq:subroutine_fixed_budget_B} comes from the assumption. As a result, we know that $d_{k} \geq d_\star$ since $d_{k} \in [\widetilde D]$ is selected as the largest integer such that $g_k(d_{k}) \leq B$. 

\textbf{Step 1.2: Concentration and error probability.} Let $\curly*{x_1, \ldots, x_{T^\prime}} $ be the arms pulled at iteration $k$ and $\curly*{r_1, \ldots, r_{T^\prime}}$ be the corresponding rewards. Let $\widehat{\theta}_k = A_k^{-1} b_k \in \R^{d_k}$ where $A_k = \sum_{i=1}^{T^\prime} \psi_{d_k}(x_i) \psi_{d_k}(x_i)^\top$, and $b_k =  \sum_{i=1}^{T^\prime} \psi_{d_k}(x_i) b_i$. Since $d_k \geq d_\star$ and the model is well-specified, we can write $r_i = \ang*{\theta_\star, x_i} + \xi_i = \ang*{\psi_{d_k}(\theta_\star), \psi_{d_k}(x_i)} + \xi_i$, where $\xi_i$ is i.i.d. generated zero-mean Gaussian noise with variance $1$. Similarly as analyzed in \cref{eq:subroutine_fixed_confidence_event}, we have 
\begin{align}
    \P \paren{ \forall y \in \cY(\psi_{d_k}(\widehat \cS_k)) , \abs{ \ang{ y, \widehat \theta_k - \psi_{d_k}(\theta_\star)  } } \leq \norm{y}_{A_k^{-1}} \sqrt{2 \log \paren{ {\abs*{\widehat \cS_k}^2}/{\delta_k} }}  } \geq 1-\delta_k. \label{eq:subroutine_fixed_budget_event}
\end{align}
By setting $\max_{y \in \psi_{d_k}(\widehat \cS_k)} \norm{y}_{A_k^{-1}} \sqrt{2 \log \paren{ {\abs*{\widehat \cS_k}^2}/{\delta_k} }} = 2^{-k}$, we have 
\begin{align}
    \delta_k & = \abs*{\widehat \cS_k}^2 \exp \paren{ - \frac{1}{ 2 \cdot 2^{2k} \, \max_{y \in \psi_{d_k}(\widehat \cS_k)} \norm{y}_{A_k^{-1}}^2 } } \nonumber \\
    & \leq \abs*{\widehat \cS_k}^2 \exp \paren{ - \frac{T^\prime}{ 2 \cdot 2^{2k} \, (1 + \zeta) \, \iota(\cY(\psi_{d_k} (\widehat \cS_k))) } } \label{eq:sub_fixed_budget_rounding} \\
    & \leq \abs*{ \cZ }^2 \exp \paren{ - \frac{ T}{ 1024 \, n \, \rho_{d_\star}^\star } } \label{eq:sub_fixed_budget_error_bound},
\end{align}
where \cref{eq:sub_fixed_budget_rounding} comes from the guarantee of the rounding procedure \cref{eq:rounding}; and \cref{eq:sub_fixed_budget_error_bound} comes from combining the following facts: (1) $2^{2k} \, \iota(\cY(\psi_{d_k} (\widehat \cS_k))) \leq B \leq 128 \rho_{d_\star}^\star$; (2) $T^\prime \geq T/n - 1 \geq T/2n$ (note that $T/n \geq r_{d_\star}(\zeta) + 1 \implies T/n \geq 2$ since $r_{d_\star}(\zeta) \geq 1$); (3) $\widehat \cS_k \subseteq \cZ$ and (4) consider some $\zeta \leq 1$ ($\zeta$ only affects constant terms).

\textbf{Step 1.3: Correctness.} We prove $z_\star \in \widehat \cS_{k+1} \subseteq \cS_{k+1}$ under the good event analyzed in \cref{eq:subroutine_fixed_budget_event}.

\textbf{Step 1.3.1: $z_\star \in \widehat \cS_{k+1}$.} For any $\widehat z \in \widehat \cS_k$ such that $\widehat z \neq z_\star$, we have 
\begin{align*}
    \ang{ \psi_{d_k}(\widehat z) - \psi_{d_k}(z_\star) , \widehat \theta_k } & \leq \ang{ \psi_{d_k}(\widehat z) - \psi_{d_k}(z_\star) , \psi_{d_k} (\theta_\star) } + 2^{-k} \\
    & = h(\widehat z ) - h(z_\star) + 2^{-k} \\
    & < 2^{-k}.
\end{align*}
As a result, $z_\star$ remains in $\widehat \cS_{k+1}$ according to the elimination criteria.

\textbf{Step 1.3.2: $\widehat \cS_{k+1} \subseteq \cS_{k+1}$.} Consider any $z \in \widehat \cS_k \cap \cS_{k+1}^c$, we know that $\Delta_z \geq 2\cdot 2^{-k}$ by definition. Since $z_\star \in \widehat \cS_k$, we then have 
\begin{align}
    \ang*{ \psi_{d_k}(z_\star) - \psi_{d_k}(z) , \widehat \theta_k } & \geq \ang{ \psi_{d_k}(z_\star) - \psi_{d_k}(z) , \psi_{d_k} (\theta_\star) } - 2^{-k} \nonumber \\
    & = h(z_\star) - h(z) - 2^{-k} \nonumber \\ 
    & \geq 2 \cdot 2^{-k} - 2^{-k} \nonumber \\
    & = 2^{-k}.
\end{align}
As a result, we have $z \notin \widehat \cS_{k+1}$ and $\widehat \cS_{k+1} \subseteq \cS_{k+1}$.

To summarize, we prove the induction at iteration $k+1$, i.e.,
\begin{align*}
    \P \paren{ \cE_{k+1} \mid \cap_{i< k+1} \cE_{i} } & \geq 1 - \delta_k. \nonumber
\end{align*}

\textbf{Step 2: The error probability.} Let $\cE = \cap_{i=1}^{n+1} \cE_{i}$ denote the good event, we then have 
\begin{align}
    \P \paren{ \cE } & = \prod_{k=1}^{n+1} \P \paren{ \cE_k \mid \cE_{k-1} \cap \dots \cap \cE_1 } \nonumber \\
    & = \prod_{k=1}^{n+1} \paren{1 - \delta_k} \nonumber \\
    & \geq 1 - \sum_{i=1}^{n+1} \delta_k \label{eq:subroutine_fixed_budget_induction} \\
    & \geq 1- n \abs*{ \cZ }^2 \exp \paren{ - \frac{ T}{ 640 \, n \, \rho_{d_\star}^\star } } \nonumber,
\end{align}
where \cref{eq:subroutine_fixed_budget_induction} can be proved using a simple induction. 
\end{proof}

\subsection{Proof of \cref{thm:doubling_fixed_budget}}

\thmDoublingFixedBudget*

\begin{proof}
The proof is decomposed into three steps: (1) locate a good subroutine in the pre-selection step; (2) bound error probability in the validation step; and (3) analyze the total error probability. Some preliminaries are analyzed as follows.

We note that both pre-selection and validation steps use budget less than $T$: in the pre-selection phase, each outer loop indexed by $i$ uses budget less than $T/p$ and there are $p$ such outer loops; it's also clear that the validation steps uses at most $T$ budget. We notice that $p \leq \log_2 T$ since $p \cdot 2^p \leq T$; and $q_i \leq \log_2 T$ since $q_i \cdot 2^{q_i} \leq T/p B_i \leq T$. As a result, at most $(\log_2 T)^2$ subroutines are invoked in \cref{alg:doubling_fixed_budget}, and each subroutine is invoked with budget $T^{\prime \prime} \geq T/(\log_2 T)^2$.

\textbf{Step 1: The good subroutine.} Consider
\begin{align*}
    i_\star \ldef \ceil{ \log_2 \paren{ 64 \rho_{d_\star}^\star } } \quad \text{and} \quad 
    j_\star \ldef \ceil{ \log_2 \paren{ \log_2 \paren{ 2/\Delta_{\min} } } }.
\end{align*}
One can easily see that $64 \rho_{d_\star}^\star \leq B_{i_\star} \leq 128 \rho_{d_\star}^\star$ and $n_{j_\star} \geq \log_2 (2/\Delta_{\min})$. Thus, once a subroutine is invoked with $(i_\star, j_\star)$ and $T^{\prime \prime}/ n_{j_\star} \geq r_{d_\star}(\zeta) + 1$, \cref{lm:subroutine_fixed_budget} guarantees to output the optimal arm with error probability at most 
\begin{align}
     \log_2 (4/\Delta_{\min}) \abs*{ \cZ }^2 \exp \paren{ - \frac{ T}{1024 \, \log_2 (4/\Delta_{\min}) \, \rho_{d_\star}^\star } } \label{eq:sub_fixed_budget_good_sub_error_prob}.
\end{align}
We next show that for sufficiently large $T$, one can invoke the subroutine with $(i_\star, j_\star)$ and $T^{\prime \prime}/ n_{j_\star} \geq r_{d_\star}(\zeta) + 1$. 

We clearly have $p \geq i_\star$ as long as $T \geq \log_2(128 \rho_{d_\star}^\star) \, 128 \rho_{d_\star}^\star$. Focusing on the outer loop with index $i_\star$, we have $q_{i_\star} \geq j_\star$ as long as 
\begin{align*}
    \log_2 \paren{2 \log_2 (2 /\Delta_{\min})} \cdot \paren{2 \log_2 (2/\Delta_{\min})} \leq T^\prime / B_{i_\star},
\end{align*}
Since $T^\prime/B_{i_\star} \geq T/ (128 \rho_{d_\star}^\star \log_2 T )$, we have $q_{i_\star} \geq j_\star$ as long as $T$ is such that
\begin{align}
    T \geq 256 \, \log_2 \paren{2 \log_2 (2 /\Delta_{\min})} \cdot \log_2 (2/\Delta_{\min}) \cdot \rho_{d_\star}^\star \cdot \log_2 T. \label{eq:budget_requirement_1}
\end{align}
Since $T^{\prime \prime} \geq T/(\log_2 T)^2$, we have $T^{\prime \prime}/ n_{j_\star} \geq r_{d_\star}(\zeta) + 1$ as long as $T$ is such that
\begin{align}
    T \geq (r_{d_\star}(\zeta) + 1) \cdot \log_2(4/\Delta_{\min}) \cdot (\log_2 T)^2. \label{eq:budget_requirement_2}
\end{align}
According to \cref{lm:relation_log}, \cref{eq:budget_requirement_1} and \cref{eq:budget_requirement_2} can be satisfied when 
\begin{align*}
    T = \widetilde \Omega \paren{ \log_2(1/\Delta_{\min}) \max \curly*{\rho_{d_\star}^
    \star, r_{d_\star}(\zeta)} },
\end{align*}
where lower order terms with respect to $\log_2(1/\Delta_{\min})$, $\rho_{d_\star}^\star$ and $r_{d_\star}(\zeta)$ are hidden in the $\widetilde \Omega$ notation.

\textbf{Step 2: The validation step.} We have $\abs*{\cA} \leq (\log_2 T)^2$ since there are at most $(\log_2 T)^2$ subroutines and each subroutine outputs one arm. We view each $x \in \cA$ as individual arm and pull it $\floor*{T/\abs*{\cA}} \geq T/(\log_2 T)^2 -1 \geq T/2(\log_2 T)^2$ (as long as $T \geq 2 (\log_2 T)^2$) times. We use $\widehat h(x)$ to denote the empirical mean of $h(x)$. Applying Hoeffding's inequality with a union bound leads to the following concentration result
\begin{align*}
    \P \paren{ \forall x \in \cA: \abs*{\widehat h(x) - h(x)} \geq \Delta_{\min} /2  } \leq 2 (\log_2 T)^2 \exp \paren{ - \frac{T}{8 (\log_2 T)^2/ \Delta_{\min}^2} }
\end{align*}
Thus, as long as $z_\star \in \cA$ is selected in $\cA$ from the pre-selection step, the validation step correctly output $z_\star$ with error probability at most 
\begin{align}
    2 (\log_2 T)^2 \exp \paren{ - \frac{T}{8 (\log_2 T)^2/ \Delta_{\min}^2} }
    \label{eq:sub_fixed_budget_validation_error_prob}.
\end{align}

\textbf{Step 3: Total error probability.} Combining \cref{eq:sub_fixed_budget_good_sub_error_prob} with \cref{eq:sub_fixed_budget_validation_error_prob}, we know that
\begin{align*}
    \P \paren{\widehat z_\star \neq z_\star} \leq & \log_2 (4/\Delta_{\min}) \abs*{ \cZ }^2 \exp \paren{ - \frac{ T}{ 1024 \, \log_2 (4/\Delta_{\min}) \, \rho_{d_\star}^\star } } \\
    &+ 2 (\log_2 T)^2 \exp \paren{ - \frac{T}{8 (\log_2 T)^2/ \Delta_{\min}^2} }.
\end{align*}
Furthermore, if there exists universial constants such that $\max_{x \in \cX} \norm{\psi_{d_\star}(x)}^2 \leq c_1$ and $\min_{z \in \cZ} \norm{\psi_{d_\star}(z_\star) - \psi_{d_\star}(z)}^2 \geq c_2$, \cref{lm:psi_ub_lb} implies that $1/\Delta_{\min}^2 \leq c_1 \rho_{d_\star}^\star/ c_2$. We thus have 
\begin{align*}
    \P  \paren{\widehat z_\star \neq z_\star} =  O \paren{ \max \curly*{ \log_2(1/\Delta_{\min}) \abs*{\cZ}^2, (\log_2 T)^2 } \cdot \exp \paren{ - \frac{c_2 T}{\max \curly*{ \log_2(1/\Delta_{\min}), (\log_2 T)^2 } c_1 \rho_{d_\star}^\star} } }.
\end{align*}
\end{proof}

\section{OMITTED PROOFS FOR SECTION \ref{sec:misspecification}}

\subsection{Omitted Proofs for Propositions}

Some of the propositions are borrowed from \citet{zhu2021pure}, we present detailed proofs here for completeness.

\propNonIncreasingMisspecification*

\begin{proof}
Consider any $1 \leq d < d^\prime \leq D$. Suppose
\begin{align*}
    \theta^d \in \argmin_{\theta \in \R^D} \max_{x \in \cX \cup \cZ} \abs*{h(x) - \ang*{\psi_d(\theta), \psi_d(x)}}.
\end{align*}
Since $\psi_d(\theta^d)$ only keeps the first $d$ component of $\theta^d$, we can choose $\theta^d$ such that it only has non-zero values on its first $d$ entries. As a result, we have $\ang*{\psi_d(\theta^d), \psi_d(x)} = \ang*{\psi_{d^\prime}(\theta^d), \psi_{d^\prime}(x)}$, which implies that $\widetilde \gamma(d^\prime) \leq \widetilde \gamma (d)$.
\end{proof}

\propRhoRelation*

\begin{proof}
To relate $\rho_d^\star(\epsilon)$ with $\widetilde \rho_d^\star(\epsilon)$, we only need to relate $\max \curly{ h(z_\star)- h(z), \epsilon}$ with $\max \curly{ \ang{\psi_d(z_\star) - \psi_d(z), \theta_\star^d}, \epsilon}$. From \cref{eq:mis_level} and the fact that $\epsilon \geq \widetilde \gamma (d)$, we know that 
\begin{align*}
    \ang{\psi_d(z_\star) - \psi_d(z), \theta_\star^d} \leq h(z_\star)- h(z) + 2 \widetilde \gamma (d) \leq h(z_\star)- h(z) + 2 \epsilon \leq 3 \max \curly{ h(z_\star)- h(z), \epsilon},
\end{align*}
and thus 
\begin{align*}
    \max \curly{ \ang{\psi_d(z_\star) - \psi_d(z), \theta_\star^d}, \epsilon} \leq 3 \max \curly{ h(z_\star)- h(z), \epsilon}.
\end{align*}
As a result, we have $\rho_d^\star(\epsilon) \leq 9 \widetilde \rho_d^\star (\epsilon)$.

When $\widetilde \gamma (d) < \Delta_{\min}/2$, we know that $z_\star$ is still the best arm in the perfect linear bandit model (without misspecification) $\widetilde h(x) = \ang*{\psi_d(x), \psi_d(\theta_\star^d)}$. Thus, $\widetilde \rho_d^\star(0)$ represents the complexity measure, in the corresponding linear model, for best arm identification.
\end{proof}

\begin{restatable}[\citet{zhu2021pure}]{proposition}{propGammaUpperBound}
\label{prop:gamma_upper_bound}
The following inequalities hold:
\begin{align*}
    \gamma(d) \leq \paren*{16 + 16 \sqrt{(1+\zeta) d}} \widetilde \gamma(d) = O(\sqrt{d} \, \widetilde \gamma (d)).
\end{align*}
\end{restatable}

\begin{proof}
We first notice that 
\begin{align}
    \iota \paren*{ \cY( \psi_d(\cS_k)) } & = \inf_{\lambda \in \simp_{\cX}} \sup_{y \in \cY( \psi_d(\cS_k))} \norm{y}_{A_{d}(\lambda)^{-1}}^2 \nonumber \\
    & \leq \inf_{\lambda \in \simp_{\cX}} \sup_{y \in \cY( \psi_d(\cX))} \norm{y}_{A_{d}(\lambda)^{-1}}^2 \nonumber \\
    & \leq \inf_{\lambda \in \simp_{\cX}} \sup_{ x \in \cX} 4 \norm{\psi_d(x)}_{A_{d}(\lambda)^{-1}}^2 \nonumber \\
    & = 4d \label{eq:prop_gamma_upper_bound_kw},
\end{align}
where \cref{eq:prop_gamma_upper_bound_kw} comes from Kiefer-Wolfowitz theorem \citep{kiefer1960equivalence}. We then have 
\begin{align*}
    \paren{2 + \sqrt{(1+\zeta) \iota \paren*{ \cY( \psi_d(\cS_k)) }} } \widetilde \gamma(d) \leq \paren{2 + \sqrt{(1+\zeta) 4 d} } \widetilde \gamma(d).
\end{align*}
As a result, we can always find a $n \in \N$ such that 
\begin{align*}
    2^{-n}/2 \leq 2 \, \paren{2 + \sqrt{(1+\zeta) 4 d} } \widetilde \gamma(d),
\end{align*}
and 
\begin{align*}
    \paren{2 + \sqrt{(1+\zeta) \iota \paren*{ \cY( \psi_d(\cS_k)) }} } \widetilde \gamma(d) \leq \paren{2 + \sqrt{(1+\zeta) 4 d} } \widetilde \gamma(d) \leq 2^{-k}/2, \forall k \leq n.
\end{align*}
This leads to the fact that
\begin{align*}
    \gamma(d) \leq 8 \, \paren{2 + \sqrt{(1+\zeta) 4 d} } \widetilde \gamma(d),
\end{align*}
which implies the desired result.
\end{proof}

\begin{restatable}{proposition}{propRoundNumber}
\label{prop:round_number}
If $\gamma(d) \leq \epsilon$, we have 
\begin{align*}
    \paren{2 + \sqrt{(1+\zeta) \iota \paren*{ \cY( \psi_d(\cS_k)) }} } \widetilde \gamma(d) \leq 2^{-k}/2, \forall k \leq \ceil*{\log_2(2/\epsilon)}.
\end{align*}
\end{restatable}

\begin{proof}
Suppose $\gamma(d) = 2 \cdot 2^{-\widetilde n}$ for a $\widetilde n \in \N$. Since $\gamma(d) \leq \epsilon$, we have $\widetilde n \geq \log_2(2/\epsilon)$. Since $\widetilde n \in \N$, we know that $\widetilde n \geq \ceil*{\log_2(2/\epsilon)}$. The desired result follows from the definition of $\gamma(d)$.
\end{proof}

\subsection{Omitted Materials for the Fixed Confidence Setting with Misspecification}

\subsubsection{Omitted Algorithms}
\label{app:alg_misspecification}

\begin{algorithm}[H]
	\caption{\gemsm Gap Elimination with Model Selection with Misspecification (Fixed Confidence)}
	\label{alg:subroutine_fixed_confidence_mis_gen} 
	\renewcommand{\algorithmicrequire}{\textbf{Input:}}
	\renewcommand{\algorithmicensure}{\textbf{Output:}}
	\begin{algorithmic}[1]
		\REQUIRE  Number of iterations $n$, budget for dimension selection $B$ and confidence parameter $\delta$.
		\STATE Set $\widehat \cS_1 = \cZ$.
		\FOR {$k = 1, 2, \dots, n$}
		\STATE Set $\delta_k = \delta/k^2$.
		\STATE Define function $g_k(d) \ldef \max \curly*{ 2^{2k} \, \iota_{k,d}, r_d(\zeta) }$, where $\iota_{k,d} \ldef \iota(\cY(\psi_d(\widehat \cS_k)))$. 
		\STATE Get $d_k = \opt(B, D, g_k(\cdot))$, where $d_k \leq D$ is largest dimension such that $g_k(d_k) \leq B$ (see \cref{eq:opt_d_selection} for the detailed optimization problem). Set $\lambda_k$ be the optimal design of the optimization problem
		$$\inf_{\lambda \in \simp_{\cX}} \sup_{z, z^\prime \in \widehat \cS_k} \norm*{\psi_{d_k}(z)-\psi_{d_k}(z^\prime)}^2_{A_{d_k}(\lambda)^{-1}}, \ \text{ and }  N_k = \ceil*{g(d_k) 8(1 + \zeta)  \log(\abs*{\widehat \cS_k}^2/\delta_k)}.$$
		\STATE Get allocation $\curly*{x_1, \ldots, x_{N_k} } = \round (\lambda_k,N_k, d_k, \zeta)$.
		\STATE Pull arms $\curly*{x_1, \ldots, x_{N_k}} $ and receive rewards $\curly*{r_1, \ldots, r_{N_k}}$.
        \STATE Set $\widehat{\theta}_k = A_k^{-1} b_k \in \R^{d_k}$ where $A_k = \sum_{i=1}^{N_k} \psi_{d_k}(x_i) \psi_{d_k}(x_i)^\top$, and $b_k =  \sum_{i=1}^{N_k} \psi_{d_k}(x_i) b_i$.
        \STATE Set $\widehat \cS_{k+1} = \widehat \cS_k \setminus \{z \in \widehat \cS_k : \exists z^\prime \text{ s.t. } \ang*{\widehat{\theta}_k, \psi_{d_k}(z^\prime) - \psi_{d_k}(z) } \geq 2^{-k} \}$.
		\ENDFOR 
		\ENSURE Any $\widehat z_\star \in \widehat \cS_{n+1}$ (or the whole set $\widehat \cS_{n+1}$ when aiming at identifying the optimal arm).
	\end{algorithmic}
\end{algorithm}

\begin{algorithm}[H]
	\caption{Adaptive Strategy for Model Selection with misspecification (Fixed Confidence)}
	\label{alg:doubling_fixed_confidence_mis_gen} 
	\renewcommand{\algorithmicrequire}{\textbf{Input:}}
	\renewcommand{\algorithmicensure}{\textbf{Output:}}
	\newcommand{\algorithmicbreak}{\textbf{break}}
    \newcommand{\BREAK}{\STATE \algorithmicbreak}
	\begin{algorithmic}[1]
		\REQUIRE Confidence parameter $\delta$.
		\STATE Randomly select a $\widehat z_\star \in \cX$ as the recommendation for the $\epsilon$-optimal arm.
		\FOR {$\ell = 1, 2, \dots$}
		\STATE Set $\gamma_\ell = 2^\ell$ and $\delta_\ell = \delta/(4\ell^3)$. Initialize an empty pre-selection set $\cA_{\ell} = \curly{}$.
		    \FOR {$i = 1, 2, \dots, \ell$}
		    \STATE Set $n_i = 2^i$, $B_i = 2^{\ell -i}$ and get $\widehat z_{\star}^i = \text{\gemsm}(n_i, B_i, \delta_\ell)$. Insert $\widehat z_{\star}^i$ into $\cA_{\ell}$.
		    \ENDFOR
		    \STATE \textbf{Validation.} Pull each arm in $\cA$ exactly $\ceil*{8 \log(2/\delta_{\ell})/\epsilon^2}$ times. Update $\widehat z_\star$ as the arm with the highest empirical mean (break ties arbitrarily).
		\ENDFOR 
	\end{algorithmic}
\end{algorithm}

\subsubsection{\cref{lm:subroutine_fixed_confidence_mis_gen} and Its Proof}

We introduce function $f:\N_+ \rightarrow \R_+$ as follows, which is also used in \cref{app:fix_budget_mis}.
\begin{align*}
    f(k) \ldef \begin{cases} 4 \cdot 2^{-k} & \text{ if } k \leq \ceil*{\log_2 (2/\epsilon)} + 1, \\
    4 \cdot \epsilon^{- \ceil*{\log_2 (4/\epsilon)}} & \text{ if } k > \ceil*{\log_2 (2/\epsilon)} + 1.\\
    \end{cases}
\end{align*}
$f(k)$ is used to quantify the optimality of the identified arm, and one can clearly see that $f(k)$ is non-increasing in $k$.

\begin{restatable}{lemma}{lmSubroutineFixedConfidenceMisGen}
\label{lm:subroutine_fixed_confidence_mis_gen}
Suppose $B \geq \max \curly*{64 \rho_{d_\star(\epsilon)}^\star( \epsilon), r_{d_\star(\epsilon)}(\zeta)}$. With probability at least $1-\delta$, \cref{alg:subroutine_fixed_confidence_mis_gen} outputs an arm $\widehat z_\star$ such that $\Delta_{\widehat z_\star} < f(n+1)$. Furthermore, an $\epsilon$-optimal arm is output as long as $n \geq {\log_2 (2/\epsilon)}$.
\end{restatable}

\begin{proof}
The logic of this proof is similar to the proof of \cref{lm:subroutine_fixed_confidence}. We additionally deal with misspecification in the proof. For fixed $\epsilon$, we use the notation $d_\star = d_\star(\epsilon)$ throughout the proof.

We consider event 
\begin{align*}
    \cE_k = \curly*{z_\star \in \widehat \cS_k \subseteq \cS_k},
\end{align*}
and prove through induction that, for $k \leq \ceil*{\log_2 (2/\epsilon)}$,
\begin{align*}
    \P \paren{ \cE_{k+1} \mid \cap_{i \leq k} \cE_i } \geq 1 - \delta_{k},
\end{align*}
where $\delta_0 \ldef 0$. Recall that $\cS_k = \curly*{z \in \cZ: \Delta_z < 4 \cdot 2^{-k}}$ (with $\cS_1 = \cZ$). For $n \geq k+1$, we have $\widehat \cS_n \subseteq \widehat \cS_{k+1}$ due to the nature of the elimination-styled algorithm, which guarantees outputting an arm such that $\Delta_{z} < f(n+1)$.

\textbf{Step 1: The induction.} We have $\curly*{ z_\star \in \widehat \cS_1 \subseteq \cS_1 }$ since $\widehat \cS_1 = \cS_1 = \cZ$ by definition for the base case (recall we assume that $\max_{z \in \cZ} \Delta_z \leq 2$). We now assume that $\cap_{i < k+1}\cE_{i}$ holds true and we prove for iteration $k+1$.

\textbf{Step 1.1: $d_k \geq d_\star$.} Since $\widehat \cS_{k} \subseteq \cS_k$, we have 
\begin{align}
    g_k(d_\star) & = \max \curly*{ {2^{2k} \iota(\cY(\psi_{d_\star}(\widehat \cS_k)))}, r_{d_\star}(\zeta) } \nonumber \\
    & \leq \max \curly*{ {2^{2k} \iota(\cY(\psi_{d_\star}(\cS_k)))}, r_{d_\star}(\zeta) } \nonumber\\
    & \leq \max \curly*{64 \rho_{d_\star}^\star (\epsilon), r_{d_\star}(\zeta)} \label{eq:subroutine_fixed_confidence_rho_stratified_mis_gen}\\
    & \leq B \label{eq:subroutine_fixed_confidence_B_mis_gen},
\end{align}
where \cref{eq:subroutine_fixed_confidence_rho_stratified_mis_gen} comes from \cref{lm:rho_stratified_eps} and \cref{eq:subroutine_fixed_confidence_B_mis_gen} comes from the assumption. As a result, we know that $d_k \geq d_\star$ since $d_k$ is selected as the largest integer such that $g_k(d_k) \leq B$. 

\textbf{Step 1.2: Concentration.} Let $\curly*{x_1, \ldots, x_{N_k}} $ be the arms pulled at iteration $k$ and $\curly*{r_1, \ldots, r_{N_k}}$ be the corresponding rewards. Let $\widehat{\theta}_k = A_k^{-1} b_k \in \R^{d_k}$ where $A_k = \sum_{i=1}^{N_k} \psi_{d_k}(x_i) \psi_{d_k}(x_i)^\top$, and $b_k =  \sum_{i=1}^{N_k} \psi_{d_k}(x_i) b_i$. Based on the definition of $\theta_\star^d \in \R^D$ and $\eta_d(\cdot)$, we can write $r_i = h(x_i) + \xi_i = \ang*{\psi_{d_k}(\theta_\star^{d_k}), \psi_{d_k}(x_i)} + \eta_{d_k}(x_i) + \xi_i$, where $\xi_i$ is i.i.d. generated zero-mean Gaussian noise with variance $1$; we also have $\abs*{\eta_{d_k}(x_i)} \leq \widetilde \gamma (d_k)$ by definition of $\widetilde \gamma (\cdot)$. For any $y \in \cY(\psi_{d_k}(\widehat \cS_k))$, we have
\begin{align}
    \abs{\ang{y, \widehat \theta_k - \psi_{d_k}(\theta_\star^{d_k})}} & = \abs{y^\top A_k^{-1} \sum_{i=1}^{N_k} \psi_{d_k}(x_i) r_i - y^\top \psi_{d_k}(\theta_\star^{d_k})} \nonumber \\
    & = \abs{y^\top A_k^{-1} \sum_{i=1}^{N_k} \psi_{d_k}(x_i) \paren{ \psi_{d_k}(x_i) ^\top \psi_{d_k}(\theta_\star^{d_k}) + \eta_{d_k}(x_i) + \xi_i} - y^\top \psi_{d_k}(\theta_\star)} \nonumber \\
    & = \abs{ y^\top A_k^{-1} \sum_{i=1}^{N_k} \psi_{d_k}(x_i) \paren{ \eta_{d_k}(x_i) + \xi_i}} \nonumber \\
    & \leq \abs{ y^\top A_k^{-1} \sum_{i=1}^{N_k} \psi_{d_k}(x_i) \eta_{d_k}(x_i) } + \abs{ y^\top A_k^{-1} \sum_{i=1}^{N_k} \psi_{d_k}(x_i) \xi_i }. \label{eq:mis_gen_sub_fixed_conf_two_terms}
\end{align}
We next bound the two terms in \cref{eq:mis_gen_sub_fixed_conf_two_terms} separately. For the first term, we have
\begin{align}
    \abs{ y^\top A_k^{-1} \sum_{i=1}^{N_k} \psi_{d_k}(x_i) \eta_{d_k}(x_i) } & \leq \widetilde \gamma(d_k) \sum_{i=1}^{N_k} \abs{ y^\top A_k^{-1} \psi_{d_k}(x_i) } \nonumber \\
    & = \widetilde \gamma (d_k) \sum_{i=1}^{N_k} \sqrt{ \paren{ y^\top A_k^{-1} \psi_{d_k}(x_i) }^2 } \nonumber \\
    & \leq \widetilde \gamma(d_k) \sqrt{N_k \sum_{i=1}^{N_k} \paren{ y^\top A_k^{-1} \psi_{d_k}(x_i) }^2 } \label{eq:mis_gen_sub_fixed_conf_jensen} \\
    & = \widetilde \gamma(d_k) \sqrt{N_k \sum_{i=1}^{N_k}  y^\top A_k^{-1} \psi_{d_k}(x_i)  \psi_{d_k}(x_i)^\top A_k^{-1} y } \nonumber \\
    & = \widetilde \gamma(d_k) \sqrt{N_k \norm{y}_{A_k^{-1}}^2 } \nonumber \\
    & \leq \widetilde \gamma(d_k) \sqrt{(1+\zeta) \iota(\cY(\psi_{d_k}(\widehat \cS_k))) } \label{eq:mis_gen_sub_fixed_conf_round} \\
    & \leq \widetilde \gamma(d_k) \sqrt{(1+\zeta) \iota(\cY(\psi_{d_k}(\cS_k))) } \label{eq:mis_gen_sub_fixed_conf_Sk}
  \end{align}
where \cref{eq:mis_gen_sub_fixed_conf_jensen} comes from Jensen's inequality; \cref{eq:mis_gen_sub_fixed_conf_round} comes from the guarantee of rounding in \cref{eq:rounding}; and \cref{eq:mis_gen_sub_fixed_conf_Sk} comes from the fact that $\widehat \cS_k \subseteq \cS_k$.

For the second term in \cref{eq:mis_gen_sub_fixed_conf_two_terms}, since $\xi_i$s are independent 1-sub-Gaussian random variables, we know that the random variable $ y^\top A_k^{-1} \sum_{i=1}^{N_k} \psi_{d_k}(x_i) \xi_i $ has variance proxy $\sqrt{\sum_{i=1}^{N_k} \paren*{y^\top A_k^{-1} \sum_{i=1}^{N_k} \psi_{d_k}(x_i)}^2 } = \norm{y}_{A_k^{-1}}$. Combining the standard Hoeffding's inequality with a union bound leads to
\begin{align}
    \P \paren{ \forall y \in \cY(\psi_{d_k}(\widehat \cS_k)) , \abs{ y^\top A_k^{-1} \sum_{i=1}^{N_k} \psi_{d_k}(x_i) \xi_i } \leq \norm{y}_{A_k^{-1}} \sqrt{2 \log \paren{ {\abs*{\widehat \cS_k}^2}/{\delta_k} }}  } \geq 1-\delta_k, \label{eq:mis_gen_sub_fixed_conf_second}
\end{align}
where we use the fact that $\abs*{ \cY(\psi_{d_k}(\widehat \cS_k)) } \leq \abs*{\widehat \cS_k}^2/2$ in the union bound. 

Putting \cref{eq:mis_gen_sub_fixed_conf_round} and \cref{eq:mis_gen_sub_fixed_conf_second} together, we have 
\begin{align}
    \P \paren{ \forall y \in \cY(\psi_{d_k}(\widehat \cS_k)) , \abs{\ang{y, \widehat \theta_k - \psi_{d_k}(\theta_\star^{d_k})}} \leq \widetilde \gamma(d_k) \iota_k +  \omega_k(y)   } \geq 1-\delta_k, \label{eq:mis_gen_sub_fixed_conf_concentration}
\end{align}
where $\iota_k \ldef  \sqrt{(1+\zeta) \iota(\cY(\psi_{d_k}(\cS_k))) }$ and $ \omega_k(y) \ldef  \norm{y}_{A_k^{-1}} \sqrt{2 \log \paren{ {\abs*{\widehat \cS_k}^2}/{\delta_k} }}$.

\textbf{Step 1.3: Correctness.} We prove $z_\star \in \widehat \cS_{k+1} \subseteq \cS_{k+1}$ under the good event analyzed in \cref{eq:mis_gen_sub_fixed_conf_concentration}. 

\textbf{Step 1.3.1: $z_\star \in \widehat \cS_{k+1}$.} For any $\widehat z \in \widehat \cS_k$ such that $\widehat z \neq z_\star$, we have 
\begin{align}
    \ang*{ \psi_{d_k}(\widehat z) - \psi_{d_k}(z_\star) , \widehat \theta_k } & \leq \ang{ \psi_{d_k}(\widehat z) - \psi_{d_k}(z_\star) , \psi_{d_k} (\theta_\star^{d_k}) } + \gamma(d_k) \iota_k + \omega_k(\psi_{d_k}(\widehat z) - \psi_{d_k}(z_\star)) \nonumber \\
    & = h(\widehat z ) -\eta_{d_k}(\widehat z) - h(z_\star) + \eta_{d_k}(z_\star) + \gamma(d_k) \iota_k + \omega_k(\psi_{d_k}(\widehat z) - \psi_{d_k}(z_\star)) \nonumber \\
    & < (2 + \iota_k) \widetilde \gamma(d_k)+ \omega_k(\psi_{d_k}(\widehat z) - \psi_{d_k}(z_\star)) \nonumber \\
    & \leq 2^{-k}/2 + 2^{-k}/2 \label{eq:mis_gen_sub_fixed_conf_optimal_arm} \\
    & = 2^{-k}, \nonumber
\end{align}
where \cref{eq:mis_gen_sub_fixed_conf_optimal_arm} comes from \cref{prop:round_number} combined with the fact that $d_k \geq d_\star$ (as shown in Step 1.1), and the selection of $N_k$ together with the guarantees in the rounding procedure \cref{eq:rounding}.

\textbf{Step 1.3.2: $\widehat \cS_{k+1} \subseteq \cS_{k+1}$.} Consider any $z \in \widehat \cS_k \cap \cS_{k+1}^c$, we know that $\Delta_z \geq 2\cdot 2^{-k}$ by definition. Since $z_\star \in \widehat \cS_k$, we then have 
\begin{align}
    \ang*{ \psi_{d_k}(z_\star) - \psi_{d_k}(z) , \widehat \theta_k } & \geq \ang{ \psi_{d_k}(\widehat z) - \psi_{d_k}(z_\star) , \psi_{d_k} (\theta_\star^{d_k}) } - \gamma(d_k) \iota_k - \omega_k(\psi_{d_k}(\widehat z) - \psi_{d_k}(z_\star)) \nonumber \\
    & = h(z_\star) - \eta_{d_k}(z_\star) - h(z) + \eta_{d_k}(z) - \gamma(d_k) \iota_k - \omega_k(\psi_{d_k}(\widehat z) - \psi_{d_k}(z_\star)) \nonumber \\ 
    & \geq 2 \cdot 2^{-k} - (2 + \iota_k)\widetilde \gamma(d_k)  - \omega_k(\psi_{d_k}(\widehat z) - \psi_{d_k}(z_\star)) \nonumber \\
    & \geq 2 \cdot 2^{-k} - 2^{-k}/2  - 2^{-k}/2 \label{eq:mis_gen_sub_fixed_conf_bad_arms} \\
    & = 2^{-k} \nonumber,
\end{align}
where \cref{eq:mis_gen_sub_fixed_conf_bad_arms} comes from a similar reasoning as appearing in \cref{eq:mis_gen_sub_fixed_conf_optimal_arm}. As a result, we have $z \notin \widehat \cS_{k+1}$ and $\widehat \cS_{k+1} \subseteq \cS_{k+1}$.

To summarize, we prove the induction at iteration $k+1$, i.e.,
\begin{align*}
    \P \paren{ \cE_{k+1} \mid \cap_{i< k+1} \cE_{i} } \geq 1 - \delta_k.
\end{align*}

\textbf{Step 2: The error probability.} The analysis on the error probability is the same as in the Step 2 in the proof of \cref{lm:subroutine_fixed_confidence}. Let $\cE = \cap_{i=1}^{n+1} \cE_{i}$ denote the good event, we then have 
\begin{align}
    \P \paren{ \cE } &  \geq 1- \delta \nonumber.
\end{align}
\end{proof}

\subsubsection{Proof of \cref{thm:doubling_fixed_confidence_mis_gen}}

\thmDoublingFixedConfidenceMisGen*

\begin{proof}
The proof is decomposed into four steps: (1) locating good subroutines; (2) guarantees for the validation step; (3) bounding error probability and (4) bounding unverifiable sample complexity. For fixed $\epsilon$, we use shorthand $d_\star = d_\star(\epsilon)$ throughout the proof.

\textbf{Step 1: The good subroutines.} Consider $B_\star = \max \curly*{64 \rho^\star_{d_\star}, r_{d_\star}(\zeta)}$ and $n_\star = \ceil*{ \log_2 (2/\epsilon) }$. For any subroutines invoked with $B_i \geq B_\star$ and $n_i \geq n_\star$, we know that, from \cref{lm:subroutine_fixed_confidence_mis_gen}, the output set of arms are those with sub-optimality gap $< \epsilon$. Let $i_\star = \ceil*{\log_2 (B_\star)}$, $j_\star = \ceil*{ \log_2 (n_\star) }$ and $\ell_\star = i_\star + j_\star$. We know that in outer loops $\ell \geq \ell_\star$, there must exists at least one subroutine invoked with $B_i = 2^{i_\star} \geq B_\star$ and $n_i = 2^{j_\star} \geq n_\star$. As a result, $\cA_{\ell}$ contains at least one $\epsilon$-optimal arm for $\ell \geq \ell_\star$. 

\textbf{Step 2: The validation step.} For any $x \in \cA_\ell$, we use $\widehat h(x)$ to denote its sample mean after  $\ceil*{8 \log(2/\delta_{\ell})/\epsilon^2}$ samples. With $1$-sub-Gaussian noise, a standard Hoeffding's inequality shows that and a union bound gives
\begin{align}
    \P \paren{\forall x \in \cA_\ell: \abs*{\widehat h(x) - h(x)} \geq \epsilon/2} \leq \ell \delta_\ell. \label{eq:mis_gen_validation}
\end{align}
As a result, a $2\epsilon$-optimal arm will be selected with probability at least $1 - \ell \delta_\ell$, as long as at least one $\epsilon$-optimal arm is contained in $\cA_\ell$.

\textbf{Step 3: Error probability.} We consider the good event where all subroutines invoked in \cref{alg:doubling_fixed_confidence} with $B_i \geq B_\star$ and (any) $n_i$ correctly output a set of arms with sub-optimality gap $<f(n_i+1)$, as shown in \cref{lm:subroutine_fixed_confidence_mis_gen}, together with the confidence bound described in \cref{eq:mis_gen_validation} in the validation step. This good event clearly happens with probability at least $1 - \sum_{\ell = 1}^\infty \sum_{i=1}^\ell 2 \delta_\ell = 1- \sum_{\ell = 1}^\infty \delta/(2 \ell^2) > 1 - \delta$, after applying a union bound argument. We upper bound the unverifiable sample complexity under this good event in the following.

\textbf{Step 4: Unverifiable sample complexity.} 
For any subroutine invoked within outer loop $\ell \leq \ell_\star$, we know, from \cref{alg:subroutine_fixed_confidence_mis_gen}, that its sample complexity is upper bounded by (note that $\abs*{\cZ}^2 \geq 4$ trivially holds true)
\begin{align*}
    N_\ell & \leq n_i \paren{ B_i \cdot \paren{10 \, \log(\abs{\cZ}^2/\delta_{\ell_\star})} + 1 } \\
    & \leq \gamma_\ell \, 11 \, \log \paren{ 4 \abs{\cZ}^2 \ell_\star^3/\delta }.
\end{align*}
The validation step within any outer loop $\ell \leq \ell_\star$ takes at most $\ell \cdot \ceil*{8 \log(2/\delta_{\ell})/\epsilon^2} \leq 9 \log(8 \ell_\star^3/\delta) \ell_\star /\epsilon^2$ samples. Thus, the total sample complexity up to the end of outer loops $\ell \leq \ell_\star$ is upper bounded by 
\begin{align*}
    N & \leq \sum_{\ell = 1}^{\ell_\star} \paren{ \ell N_\ell + \ell \cdot \ceil*{8 \log(2/\delta_{\ell})/\epsilon^2}}\\
    & \leq 11\, \log \paren{ 4 \abs{\cZ}^2 \ell_\star^3/\delta } \sum_{\ell = 1}^{\ell_\star} \ell 2^\ell +  9 \log(8 \ell_\star^3/\delta) \ell_\star^2 /\epsilon^2\\
    & \leq 22 \, \log \paren{ 4 \abs{\cZ}^2 \ell_\star^3/\delta } \ell_\star 2^{\ell_\star} + 9 \log(8 \ell_\star^3/\delta) \ell_\star^2 /\epsilon^2.
\end{align*}

By definition of $\ell_\star$, we have
\begin{align*}
    \ell_\star \leq \log_2 \paren{ 4  \log_2(4/\epsilon)  \max \curly*{64 \rho^\star_{d_\star}, r_{d_\star}(\zeta)} },
\end{align*}
and
\begin{align*}
    2^{\ell_\star} & = 2^{(i_\star + j_\star)} \\
    & \leq 4  \paren{ \log_2(2/\epsilon)+1}  \max \curly*{64 \rho^\star_{d_\star}, r_{d_\star}(\zeta)}, \\
    & =  4   \log_2(4/\epsilon)  \max \curly*{64 \rho^\star_{d_\star}, r_{d_\star}(\zeta)}. 
\end{align*}

Set $\tau_\star = \log_2(4/\epsilon) \max \curly*{\rho^\star_{d_\star}, r_{d_\star}(\zeta)}$. The unverifiable sample complexity is upper bounded by (we only consider the case when $\epsilon \leq 1$ in simplifying the bound: otherwise there is no need to prove anything since $\max_{x \in \cX} \Delta_x \leq 2$) 
\begin{align*}
    N & \leq 5632 \, \tau_\star \cdot \paren{\log_2(\tau_\star) + 8} \cdot \log \paren{ {4 \abs{\cZ}^2 \paren{\log_2(\tau_\star) + 8}^3}/{\delta} } + 9/\epsilon^2 \cdot \paren{\log_2(\tau_\star) + 8}^2 \cdot \log \paren{ {8 \paren{\log_2(\tau_\star) + 8}^3}/{\delta} } \\
    & = \widetilde O \paren{ \log_2(1/\epsilon) \max \curly*{\rho^\star_{d_\star}, r_{d_\star}(\zeta)} + 1/\epsilon^2 },
\end{align*}
where we hide logarithmic terms besides $\log(1/\epsilon)$ in the $\widetilde O$ notation.
\end{proof}

\subsubsection{Identifying the Optimal Arm under misspecification}
\label{app:BAI_misspecification}
When the goal is to identify the optimal arm under misspecification, i.e., by choosing $\epsilon = \Delta_{\min}$, one can apply \cref{alg:doubling_fixed_confidence} together with \cref{alg:subroutine_fixed_confidence_mis_gen} as the subroutine (thus removing the $1/\epsilon^2$ term in sample complexity). This combination works since, with appropriate choice of $B$, \cref{alg:subroutine_fixed_confidence_mis_gen} is guaranteed to output a subset of arms $\widehat \cS_{n+1}$ with optimality gap $<\Delta_{\min}$ when $n \geq \log_2(2/\Delta_{\min})$. This implies that $\widehat \cS = \curly*{z_\star}$ and thus the one can reuse the selection rule of \cref{alg:doubling_fixed_confidence} by recommending arms contained in the singleton set. Note that we can work with the general transductive linear bandit setting in this case, i.e., we don't require $\cZ \subseteq \cX$ anymore.

\subsection{Omitted Proofs for the Fixed Budget Setting with Misspecification}
\label{app:fix_budget_mis}

\subsubsection{\cref{lm:subroutine_fixed_budget_mis} and Its Proof}

\begin{restatable}{lemma}{lmSubroutineFixedBudgetMis}
\label{lm:subroutine_fixed_budget_mis}
Suppose $64 \rho_{d_\star (\epsilon)}^\star(\epsilon) \leq B \leq 128 \rho_{d_\star(\epsilon)}^\star (\epsilon) $ and $T/n \geq r_{d_\star(\epsilon)}(\zeta) + 1$. \cref{alg:subroutine_fixed_budget} outputs an arm $\widehat z_\star$ such that $\Delta_{\widehat z_\star} < f(n+1)$ with probability at least
\begin{align*}
    1- n \abs*{ \cZ }^2 \exp \paren{ - \frac{ T}{ 2560 \, n \, \rho_{d_\star (\epsilon)}^\star (\epsilon)} }.
\end{align*}
Furthermore, an $\epsilon$-optimal arm is output as long as $n \geq {\log_2 (2/\epsilon)}$.
\end{restatable}

\begin{proof}
The proof is similar to the proof of \cref{lm:subroutine_fixed_budget}, with main differences in dealing with misspecification. We provide the proof here for completeness. We consider event 
\begin{align*}
    \cE_k = \curly*{z_\star \in \widehat \cS_k \subseteq \cS_k},
\end{align*}
and prove through induction that, for $k \leq \ceil*{\log_2 (2/\epsilon)}$,
\begin{align*}
    \P \paren{ \cE_{k+1} \mid \cap_{i \leq k} \cE_i } \geq 1 - \delta_{k},
\end{align*}
where the value of $\curly*{\delta_k}_{k=0}^{\ceil*{\log_2 (2/\epsilon)}}$ will be specified in the proof. For $n \geq k+1$, we have $\widehat \cS_n \subseteq \widehat \cS_{k+1}$ due to the nature of the elimination-styled algorithm, which guarantees outputting an arm such that $\Delta_{z} < f(n+1)$. We use the notation $d_\star = d_\star(\epsilon)$ throughout the rest of the proof.

\textbf{Step 1: The induction.} The base case $\curly*{z_\star \in \widehat \cS_1 \subseteq \cS_1}$ holds with probability $1$ by construction (thus, we have $\delta_0 = 0$). Conditioned on events $\cap_{i=1}^k \cE_i$, we next analyze the event $\cE_{k+1}$. 

\textbf{Step 1.1: $d_k \geq d_\star$.} We first notice that $\widetilde D$ is selected as the largest integer such that $r_{\widetilde D}(\zeta) \leq T^\prime$. When $T/ n \geq r_{d_\star}(\zeta) + 1$, we have $\widetilde D \geq d_\star$ since $T^\prime \geq T/ n -1 \geq r_{d_\star}(\zeta)$. We remark here that for whatever $d_k \in [\widetilde D]$ selected, we always have $r_{d_\star}(\zeta) \leq r_{\widetilde D}(\zeta) \leq T^\prime$ and can thus safely apply the rounding procedure described in \cref{eq:rounding}.

Since $\widehat \cS_{k} \subseteq \cS_k$, we also have 
\begin{align}
    g_k(d_\star) & =  {2^{2k} \iota(\cY(\psi_{d_\star}(\widehat \cS_k)))} \nonumber \\
    & \leq {2^{2k} \iota(\cY(\psi_{d_\star}(\cS_k)))} \nonumber\\
    & \leq 64 \rho_{d_\star}^\star (\epsilon) \label{eq:subroutine_fixed_budget_rho_stratified_mis}\\
    & \leq B \label{eq:subroutine_fixed_budget_B_mis},
\end{align}
where \cref{eq:subroutine_fixed_budget_rho_stratified_mis} comes from \cref{lm:rho_stratified_eps} and \cref{eq:subroutine_fixed_budget_B_mis} comes from the assumption. As a result, we know that $d_{k} \geq d_\star$ since $d_{k} \in [\widetilde D]$ is selected as the largest integer such that $g_k(d_{k}) \leq B$. 

\textbf{Step 1.2: Concentration and error probability.} Let $\curly*{x_1, \ldots, x_{T^\prime}} $ be the arms pulled at iteration $k$ and $\curly*{r_1, \ldots, r_{T^\prime}}$ be the corresponding rewards. Let $\widehat{\theta}_k = A_k^{-1} b_k \in \R^{d_k}$ where $A_k = \sum_{i=1}^{T^\prime} \psi_{d_k}(x_i) \psi_{d_k}(x_i)^\top$, and $b_k =  \sum_{i=1}^{T^\prime} \psi_{d_k}(x_i) b_i$. Since $d_k \geq d_\star$ and the model is well-specified, we can write $r_i = \ang*{\theta_\star, x_i} + \xi_i = \ang*{\psi_{d_k}(\theta_\star), \psi_{d_k}(x_i)} + \xi_i$, where $\xi_i$ is i.i.d. generated zero-mean Gaussian noise with variance $1$. Similarly as analyzed in \cref{eq:mis_gen_sub_fixed_conf_concentration}, we have 
\begin{align}
    \P \paren{ \forall y \in \cY(\psi_{d_k}(\widehat \cS_k)) , \abs{ \ang{ y, \widehat \theta_k - \psi_{d_k}(\theta_\star)  } } \leq  \widetilde \gamma(d_k) \iota_k +  \omega_k(y)   } \geq 1-\delta_k, \label{eq:subroutine_fixed_budget_event_mis}
\end{align}
where $\iota_k \ldef  \sqrt{(1+\zeta) \iota(\cY(\psi_{d_k}(\cS_k))) }$ and $ \omega_k(y) \ldef  \norm{y}_{A_k^{-1}} \sqrt{2 \log \paren{ {\abs*{\widehat \cS_k}^2}/{\delta_k} }}$.

By setting $\max_{y \in \psi_{d_k}(\widehat \cS_k)}\norm{y}_{A_k^{-1}} \sqrt{2 \log \paren{ {\abs*{\widehat \cS_k}^2}/{\delta_k} }} = 2^{-k}/2$, we have 
\begin{align}
    \delta_k & = \abs*{\widehat \cS_k}^2 \exp \paren{ - \frac{1}{ 8 \cdot 2^{2k} \, \max_{y \in \psi_{d_k}(\widehat \cS_k)} \norm{y}_{A_k^{-1}}^2 } } \nonumber \\
    & \leq \abs*{\widehat \cS_k}^2 \exp \paren{ - \frac{T^\prime}{ 8 \cdot 2^{2k} \, (1 + \zeta) \, \iota(\cY(\psi_{d_k} (\widehat \cS_k))) } } \label{eq:sub_fixed_budget_rounding_mis} \\
    & \leq \abs*{ \cZ }^2 \exp \paren{ - \frac{ T}{ 4096 \, n \, \rho_{d_\star}^\star (\epsilon)} } \label{eq:sub_fixed_budget_error_bound_mis},
\end{align}
where \cref{eq:sub_fixed_budget_rounding_mis} comes from the guarantee of the rounding procedure \cref{eq:rounding}; and \cref{eq:sub_fixed_budget_error_bound_mis} comes from combining the following facts: (1) $2^{2k} \, \iota(\cY(\psi_{d_k} (\widehat \cS_k))) \leq B \leq 128 \rho_{d_\star}^\star (\epsilon)$; (2) $T^\prime \geq T/n - 1 \geq T/2n$ (note that $T/n \geq r_{d_\star}(\zeta) + 1 \implies T/n \geq 2$ since $r_{d_\star}(\zeta) \geq 1$);
(3) $\widehat \cS_k \subseteq \cZ$ and (4) consider some $\zeta \leq 1$ ($\zeta$ only affects constant terms).

\textbf{Step 1.3: Correctness.} We prove $z_\star \in \widehat \cS_{k+1} \subseteq \cS_{k+1}$ under the good event analyzed in \cref{eq:subroutine_fixed_budget_event_mis}.

\textbf{Step 1.3.1: $z_\star \in \widehat \cS_{k+1}$.} For any $\widehat z \in \widehat \cS_k$ such that $\widehat z \neq z_\star$, we have 
\begin{align}
    \ang*{ \psi_{d_k}(\widehat z) - \psi_{d_k}(z_\star) , \widehat \theta_k } & \leq \ang*{ \psi_{d_k}(\widehat z) - \psi_{d_k}(z_\star) , \psi_{d_k} (\theta_\star^{d_k}) } + \widetilde \gamma(d_k) \iota_k + 2^{-k}/2  \nonumber \\
    & = h(\widehat z ) -\eta_{d_k}(\widehat z) - h(z_\star) + \eta_{d_k}(z_\star) + \widetilde \gamma(d_k) \iota_k + 2^{-k}/2  \nonumber \\
    & < (2 + \iota_k) \, \widetilde \gamma(d_k)+ 2^{-k}/2  \nonumber \\
    & \leq 2^{-k}/2 + 2^{-k}/2 \label{eq:mis_sub_fixed_budget_optimal_arm} \\
    & = 2^{-k}, \nonumber
\end{align}
where \cref{eq:mis_sub_fixed_budget_optimal_arm} comes from comes from \cref{prop:round_number} combined with the fact that $d_k \geq d_\star$ (as shown in Step 1.1).
As a result, $z_\star$ remains in $\widehat \cS_{k+1}$ according to the elimination criteria.

\textbf{Step 1.3.2: $\widehat \cS_{k+1} \subseteq \cS_{k+1}$.} Consider any $z \in \widehat \cS_k \cap \cS_{k+1}^c$, we know that $\Delta_z \geq 2\cdot 2^{-k}$ by definition. Since $z_\star \in \widehat \cS_k$, we then have 
\begin{align}
    \ang*{ \psi_{d_k}(z_\star) - \psi_{d_k}(z) , \widehat \theta_k } & \geq \ang{ \psi_{d_k}(\widehat z) - \psi_{d_k}(z_\star) , \psi_{d_k} (\theta_\star^{d_k}) } - \widetilde \gamma(d_k) \iota_k - 2^{-k}/2  \nonumber \\
    & = h(z_\star) - \eta_{d_k}(z_\star) - h(z) + \eta_{d_k}(z) - \widetilde \gamma(d_k) \iota_k - 2^{-k}/2  \nonumber \\ 
    & \geq 2 \cdot 2^{-k} - (2 +  \iota_k)\widetilde \gamma(d_k)  - 2^{-k}/2  \nonumber \\
    & = 2 \cdot 2^{-k} - \gamma(d_k)  - 2^{-k}/2  \nonumber \\
    & \geq 2^{-k} \label{eq:mis_sub_fixed_budget_bad_arms},
\end{align}
where \cref{eq:mis_sub_fixed_budget_bad_arms} comes from a similar reasoning as appearing in \cref{eq:mis_sub_fixed_budget_optimal_arm}. As a result, we have $z \notin \widehat \cS_{k+1}$ and $\widehat \cS_{k+1} \subseteq \cS_{k+1}$.

To summarize, we prove the induction at iteration $k+1$, i.e.,
\begin{align*}
    \P \paren{ \cE_{k+1} \mid \cap_{i< k+1} \cE_{i} } \geq 1 - \delta_k.
\end{align*}

\textbf{Step 2: The error probability.} This step is exactly the same as the Step 2 in the proof of \cref{lm:subroutine_fixed_budget}. Let $\cE = \cap_{i=1}^{n+1} \cE_{i}$ denote the good event, we then have 
\begin{align}
    \P \paren{ \cE } &  \geq 1- n \abs*{ \cZ }^2 \exp \paren{ - \frac{ T}{ 4096 \, n \, \rho_{d_\star}^\star (\epsilon) } } \nonumber.
\end{align}
\end{proof}

\subsubsection{Proof of \cref{thm:doubling_fixed_budget_mis}}

\thmDoublingFixedBudgetMis*

\begin{proof}
The proof follows similar steps as the proof of \cref{thm:doubling_fixed_budget}. Although we are dealing with a misspecified model, guarantees derived in \cref{lm:subroutine_fixed_budget_mis} is similar to the ones in \cref{lm:subroutine_fixed_budget}. 
When $\epsilon \leq \Delta_{\min}$, the proof goes almost exactly the same as the proof of \cref{thm:doubling_fixed_budget} (with $\rho_{d_\star}^\star$ replaced by $\rho_{d_\star(\epsilon)}^\star (\epsilon)$), and \cref{alg:doubling_fixed_budget} identifies the optimal arm. When $\epsilon > \Delta_{\min}$, we additionally replace $\Delta_{\min}$ by $\epsilon$ and equally split the $2\epsilon$ slackness between selection and validation steps. We also slightly modify \cref{lm:psi_ub_lb} to an $\epsilon$-relaxed version (e.g., in the derivation of \cref{eq:psi_ub_lb_Delta_min}, select a $z^\prime \in \cZ$ with sub-optimality gap $\leq \epsilon$ and then replace $\Delta_{\min}$ by $\epsilon$). 
\end{proof}

\section{ADDITIONAL EXPERIMENT DETAILS AND RESULTS}
\label{app:experiment}

We set confidence parameter $\delta = 0.05$ in our experiments, and generate rewards with Gaussian noise $\xi_t \sim \cN(0,1)$.
We parallelize our simulations on a cluster consists of two Intel® Xeon® Gold 6254 Processors.

Similar to \citet{fiez2019sequential}, we use a Frank-Wolfe type of algorithm \citep{jaggi2013revisiting} with constant step-size $\frac{2}{k+2}$ (we use $k$ to denote the iteration counter in the Frank-Wolfe algorithm) to approximately solve optimal designs. We terminate the Frank-Wolfe algorithm when the relative change of the design value is smaller than $0.01$ or when $1000$ iterations are reached. We use the rounding procedure developed in \citet{pukelsheim2006optimal} to round continuous designs to discrete allocations (with $\zeta=1$, also see \citet{fiez2019sequential} for a detailed discussion on the rounding procedure). In the implementation of \cref{alg:doubling_fixed_confidence}, we set $\gamma_\ell = 4^\ell$, $n_i = 4^i$ and $B_i = 4^{\ell - i}$, which only affect constant terms in our theoretical guarantees. We use a binary search procedure to select $d_k$ in \cref{alg:subroutine_fixed_confidence}. 

\paragraph{Additional Experiment Results.}
We consider a problem instance with $\cX = \cZ$ being $100$ randomly selected arms from the $D$ dimensional unit sphere.
We set reward function $h(x) = \ang{\theta_\star, x}$ with $\theta_\star = [\frac{1}{1^2}, \frac{1}{2^2}, \dots, \frac{1}{d_\star^2}, 0 \dots, 0]^\top \in \R^D$. We filter out instances whose smallest sub-optimality gap is smaller than $0.08$. We set $d_\star = 5$ and vary the ambient dimension $D\in\curly{25,50,75,100}$. As in \cref{sec:experiment}, we evaluate each algorithm with success rate, (unverifiable) sample complexity and runtime. We run $100$ independent random trials for each algorithm. Due to computational burdens, we force-stop both algorithms after $50,000$ samples; we also force-stop the Frank-Wolfe algorithm when $500$ iterations are reached.

\begin{table}[H]
  \caption{Comparison of Success Rate}
  \label{tab:success_rate_2}
  \centering
  \begin{tabular}{lcccc}
    \toprule
          $D $ & $25$   & $50$ & $75$ & $100$ \\
    \midrule
    \rage    & $100\%$ & $100\%$ & $98\%$ & $95\%$  \\
    Ours    & $91\%$ & $98\%$ & $97\%$ & $98\%$   \\
    \bottomrule
  \end{tabular}
\end{table}

Success rates of both algorithms are shown in \cref{tab:success_rate_2}, and \rage shows advantages over our algorithm when $D$ is small. \cref{fig:comparison_unif} shows the sample complexity of both algorithms: Our algorithm adapts to the true dimension $d_\star$ yet \rage is heavily affected by the increasing ambient dimension $D$.

\begin{figure}[h]
    \centering
    \includegraphics[width=.4\textwidth]{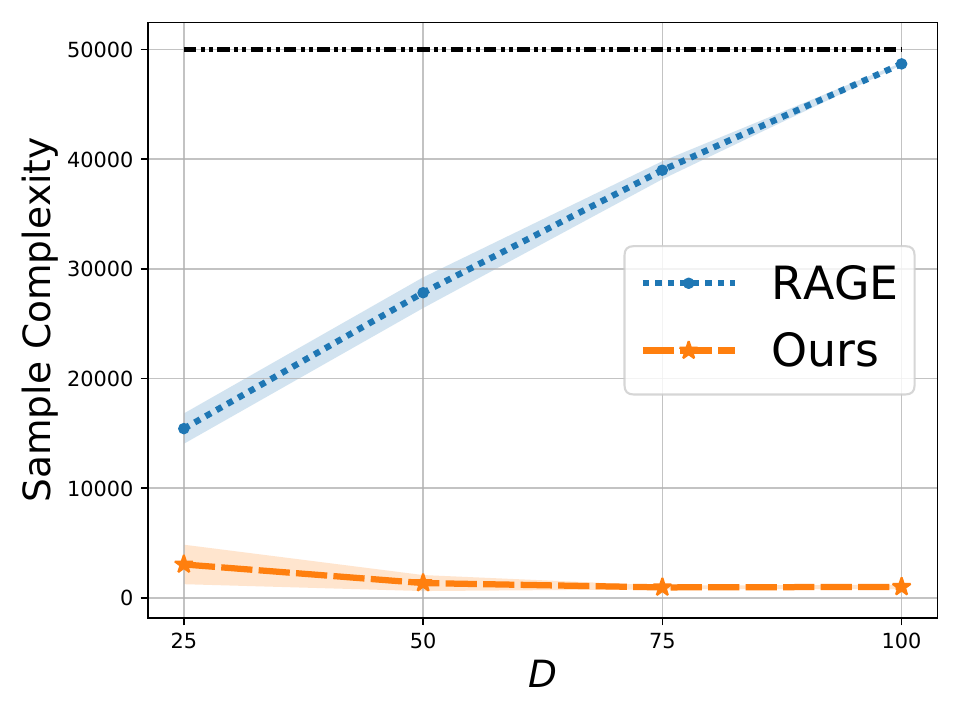}
    \caption{Comparison of Sample Complexity}
    \label{fig:comparison_unif}
\end{figure}

The runtime of both algorithms are shown in \cref{tab:runtime_unif}. \rage shows clear advantage in runtime and our algorithm suffers from computational overheads of conducting model selection.

\begin{table}[H]
  \caption{Comparison of Runtime}
  \label{tab:runtime_unif}
  \centering
  \begin{tabular}{lcccc}
    \toprule
          $\epsilon $ & $10^{-2}$   & $10^{-3}$ & $10^{-4}$ & $10^{-5}$\\
    \midrule
    \rage    & $85.99\,$s & $144.78\,$s & $249.79\,$s & $357.98\,$s \\
    Ours    & $287.09\,$s & $339.67\,$s & $489.50\,$s & $678.93\,$s  \\
    \bottomrule
  \end{tabular}
\end{table}

We remark that, for the current experiment setups with $d_\star$ and $D\in\curly{25,50,75,100}$, our algorithm does not perform well if $\theta_\star$ is chosen to be flat, e.g., $\theta_\star = [\frac{1}{\sqrt{d_\star}},\dots,\frac{1}{\sqrt{d_\star}},0,\dots,0]^\top \in \R^D$. However, we believe that one will eventually see model selection gains if $D$ is chosen to be large enough (and allowing each algorithm takes more samples before force-stopped).
One may need to overcome the computational burdens, e.g., developing practical (or heuristic-based) implementations of our algorithm and \rage, before running experiments in higher dimensional spaces. We leave large-scale evaluations for future work.

\end{document}